\documentclass[twoside,10pt]{article}

\PassOptionsToPackage{authoryear}{natbib}

\usepackage[margin=0.75in]{geometry}

\usepackage{etoolbox}
\usepackage{ifxetex,ifluatex}
\makeatletter
\newcommand{\TNRswitch}{\fontfamily{ptm}\selectfont} 
\ifxetex
  \usepackage{fontspec}
  \renewcommand{\TNRswitch}{\fontspec{Times New Roman}}
\else\ifluatex
  \usepackage{fontspec}
  \renewcommand{\TNRswitch}{\fontspec{Times New Roman}}
\fi\fi
\def\name{\normalsize\mdseries\TNRswitch}
\patchcmd{\@maketitle}{\Large\bf}{\Large\mdseries\TNRswitch}{}{}%
\makeatother

\usepackage[abbrvbib]{jmlr2e}

\usepackage{amsmath}

\makeatletter
\let\p@lemma\relax
\let\p@definition\relax
\let\p@proposition\relax
\let\p@remark\relax
\makeatother
\newtheorem{lemma}{Lemma}
\newtheorem{definition}{Definition}
\newtheorem{proposition}{Proposition}
\newtheorem{remark}{Remark}

\ShortHeadings{Transformers in Distribution Regression}{Liu and Zhou}
\firstpageno{1}

\begin{document}

\title{Generalization Analysis of Transformers in Distribution Regression}

\author{\name Peilin Liu\,$^{1}$\thanks{Email: peilin.liu@sydney.edu.au}
       \quad
       \name Ding-Xuan Zhou\,$^{2}$\thanks{Email: dingxuan.zhou@sydney.edu.au} \\
       \vspace{1em}
       \addr $^{1,2}$School of Mathematics and Statistics,
       University of Sydney, Sydney, NSW 2006, Australia
       }

\maketitle
\pagestyle{plain}

\begin{abstract}%
In recent years, models based on the Transformer architecture have seen widespread applications and have become one of the core tools in the field of deep learning. Numerous successful and efficient techniques, such as parameter-efficient fine-tuning and efficient scaling, have been proposed surrounding their applications to further enhance performance. However, the success of these strategies has always lacked the support of rigorous mathematical theory. 

To study the underlying mechanisms behind Transformers and related techniques, we first propose a Transformer learning framework motivated by distribution regression, with distributions being inputs, connect a two-stage sampling process with natural language processing, and present a mathematical formulation of the attention mechanism called attention operator. We demonstrate that by the attention operator, Transformers can compress distributions into function representations without loss of information. Moreover, with the advantages of our novel attention operator, Transformers exhibit a stronger capability to learn functionals with more complex structures than convolutional neural networks and fully connected networks. Finally, we obtain a generalization bound within the distribution regression framework. Through the aforementioned theoretical results, we further discuss some successful techniques emerging with large language models (LLMs), such as prompt tuning, parameter-efficient fine-tuning, and efficient scaling. We also provide theoretical insights behind these techniques within our novel analysis framework.
\end{abstract}


\section{Introduction}

Transformers \citep{vaswani2017attention,zhou2021informer,liu2021swin,choromanski2020rethinking,qincosformer} have undeniably become a fundamental component of modern deep learning models, extending the influence beyond the realms of natural language processing (NLP) and computer vision (CV). Transformer-based large models like GPT 4 \citep{OpenAI_4_2023},  demonstrate remarkable capabilities to process multimodal inputs with texts and images, and scientific research tools like AlphaFold \citep{jumper2021highly} are created to explore the patterns hidden in complex biological data. With the rapid developments of deep learning methods, numerous techniques for Transformers have been proposed to enhance the performance of LLMs across diverse applications. For example, techniques such as prompt tuning \citep{lester2021power,jia2022visual} and the integration of adapter modules \citep{houlsby2019parameter,hulora} are employed to adapt a pretrained LLM to new tasks at a low computational cost; As the size of the training text corpora increases dramatically, the network complexity can be efficiently scaled up using mixture of expert methods \citep{eigen2013learning,shazeer2016outrageously,jiang2024mixtral} for superior performance. Despite the impressive success in practical applications, there remains a deficiency in theoretical frameworks to demonstrate the reasons why Transformer-based models and those techniques work efficiently across diverse domains.

In recent years, mathematical theories around deep fully connected networks (FNNs) \citep{yarotsky2017error,schmidt2020nonparametric} and deep convolutional neural networks (CNNs) \citep{zhou2018deep,zhou2020universality,mao2023approximating} have been established to investigate their approximation and generalization abilities. However, for transformer-based networks, due to the complex input data structures and network architectures, it is challenging to build a theoretical framework to study the transformer structures and the phenomena observed in practical applications. In this paper, we establish a rigorous mathematical framework to demonstrate the learning capabilities of Transformers from a viewpoint of distribution regression, and also provide theoretical foundations and justifications for those practical techniques with Transformers. 

In the history of NLP, modeling problems in NLP with probabilistic tools is a classical approach. Here, we utilize a two-stage sampling process in distribution regression to formulate problems. In our distribution regression model, the inputs are distribution samples on the space $(P(\Omega), \gamma _{k})$, where $P(\Omega)$ denotes the set of all Borel probability measures defined on a compact subset $\Omega$ of $\mathbb{R}^{d}$ and $\gamma _{k}$ is a kernel embedding distance, also known as the maximum mean discrepancy (MMD) \citep{muandet2017kernel}. However, we assume that the distribution samples cannot be observed directly and that our observations are the data generated by a two-stage sampling process. In the first stage of sampling, a dataset $D=\{ (\mu _{i},y_{i}) \}_{i=1}^{m_{1}}$ is $i.i.d.$ sampled from a meta Borel distribution $\rho$ on $\mathcal{U} \times \mathcal{Y}$, where $\mathcal{U}=P(\Omega)$ and $\mathcal{Y}=\mathbb{R}$ is the output space. In the second stage of sampling, the dataset is $\hat{D}=\{ (\{ x_{i,j} \}_{j=1}^{m_{2,i}},y_{i}) \}_{i=1}^{m_{1}}$, where $\{ x_{i,j} \in \Omega \}_{j=1}^{m_{2,i}}$ are $i.i.d.$ sampled from the probability measure $\mu_{i}$, one of the first stage samples. We denote the empirical distribution of $\mu _{i}$ by $\hat{\mu}_{i}^{m_{2,i}}=\frac{1}{m_{2,i}} \sum_{j=1}^{m_{2,i}}\delta _{x_{i,j}}$, where $\delta _{x}$ is a Dirac measure. Then the second-stage sampling dataset can be denoted by $\hat{D}=\{ (\hat{\mu}_{i}^{m_{2,i}},y_{i}) \}_{i=1}^{m_{1}}$. By choosing an appropriate hypothesis space $\mathcal{H}$, the distribution regression scheme can be described as 

\begin{align} \label{eq:var}
   \varphi_{\hat{D},\mathcal{H}}:= \arg \min _{\varphi \in \mathcal{H} } \frac{1}{m_{1}} \sum_{i=1}^{m_{1}} \big(\varphi(\hat{\mu}_{i}^{m_{2,i}})-y_{i}\big)^{2}. 
\end{align}

The learning algorithm (\ref{eq:var}) for distribution regression is similar with the setting of domain generalization problems \citep{blanchard2021domain, holzleitner2024domain}. Both learning algorithms use the empirical distributions as prediction inputs. However, the key difference is that hypothesis functions in domain generalization also takes a sample as an input besides the target empirical distribution, and the learning scheme is formally defined as $\varphi'_{\hat{D},\mathcal{H}}:= \arg \min _{\varphi' \in \mathcal{H} } \frac{1}{m_{1}} \sum_{i=1}^{m_{1}}\frac{1}{m_{2,i}} \sum_{j=1}^{m_{2,i}} \big(\varphi'(\hat{\mu}_{i}^{m_{2,i}}, x_{i,j})-y_{i,j}\big)^{2}$ where $\mathcal{H}$ is a reproducing kernel Hilbert space defined on $P(\Omega) \times \Omega$, which is inconsistent with the sequential modelling case in NLP. For the case of NLP, the two-stage sampling process can be illustrated as follows: the first-stage sampling takes distribution samples at the semantic sentence level and the second stage samples tokens at the word level to express semantic meaning by forming a sentence.

Recently, some progress has been achieved in two-stage distribution regression with neural networks \citep{shi2023learning,yu2023deep}, none of which, however, matches the architecture of Transformers. Prior to the era of the prominence of neural networks, a well-known approach was based on kernel mean embedding techniques and kernel ridge regression \citep{szabo2016learning,fang2020optimal,yu2021robust}. All these works consider a regularized empirical risk minimization algorithm and that the regression function belongs to a function space characterized by the integral operator induced by a kernel function. Under assumptions on regularization of the function space and integral operator techniques, some nice generalization bounds are obtained. Afterward, the study of distribution regression focused on the application of neural networks. \citep{shi2023learning} and \citep{yu2023deep} proposed network architectures with FNNs and deep CNNs to learn the two-stage distribution regression respectively. These works metrize $P(\Omega)$ with a Wasserstein distance \citep{villani2009optimal} $W_{p}$ and learn functionals with only polynomial features. Yet, their methods don't integrate information of the input domain $(P(\Omega),W_{p})$ into the network structure and suffer from a potential information loss by only encoding polynomial features. To this end, we propose a two-stage distribution regression framework with Transformer-based network structures, which combines both advantages of the classical kernel embedding techniques and the neural network methods.

In this work, we investigate the learning capabilities of Transformer-based networks in a two-stage distribution regression framework. Our main contributions are listed as follows:
\begin{itemize}
    \item We first utilize a two-stage sampling process to understand the processing of Transformers with natural languages. We also propose a novel operator called attention operator to study the behavior of the attention layers in Transformers. We also prove that the attention operator can embed distributions into function representations, without any loss of information.
    \item We then introduce the architecture of Transformer encoders based on our novel attention operator and establish a rigorous distribution regression framework for Transformer-based networks. The approximation rate and generalization bound for the distribution regression problem are obtained, which exhibits the remarkable expressivity of Transformers in learning more diverse features than FNNs and CNNs.
    \item We provide a theoretical intuition for the choice of query sets in practical applications. There exists a universal choice for various tasks, though it may suffer from the curse of dimension in high-dimensional cases. This result justifies some task-specific tuning methods and cross-modal alignment tricks.  
    \item We establish theoretical insights for the design of FNN layers based on approximation and generalization analysis of Transformer encoders. We show that task-specific features are learned by the trainable FNN layer with a fixed attention layer, which provides theoretical foundations for adapter tuning with pretrained LLMs. We also illustrate that the complexity of the FNN layer should scale up with the training size of the text corpora to achieve a great generalization performance.  

\end{itemize}

The remainder of the paper is organized as follows. Section 2 provides the motivation and basic definitions of the learning problem. Within this section, subsection 2.1 introduces the basic structure of the vanilla Transformer \citep{vaswani2017attention} and gives the formal definition of our self-attention operator. Subsection 2.2 demonstrates the definitions of $(P(\Omega), \gamma _{k})$ and the structure of Transformer encoders for distribution regression. Subsequently, Section 3 contains the main results. First, we show in subsection 3.1 an approximation rate of a functional class induced by Barron functionals, by the Transformer-based network, then establish an oracle inequality and finally obtain a generalization bound for distribution regression in subsection 3.2. Based on the aforementioned theoretical results, we explain the principles behind the successful practical strategies, such as prompt tuning, adapters, and efficient scaling in Section 4. Finally, Section 5 presents the proof details of the main results.

\section{Motivation and Definitions}

In this section, we start with the attention operator, a fundamental component in our network structure, which is motivated by the self-attention layers in the vanilla Transformer.  The section then continues with the basic settings for the generalization analysis of distribution regression, focusing on the metric of the input space and the definition of our novel Transformer encoder for distribution regression.

\subsection{Attention operator}
The original Transformer was proposed for machine translation, consisting of an encoder and a decoder. However, with recent developments, it has become quite common to apply only a decoder (e.g., GPT \citep{radford2018improving}) or an encoder (e.g., BERT \citep{devlin2019bert}) in practical applications. In this work, we focus our analysis on a shallow Transformer encoder, which could be described as a composition of a self-attention layer and a position-wise fully connected layer. 
Now we present a mathematical definition of the Transformer introduced by \citep{vaswani2017attention}. Let $Q \in \mathbb{R}^{n \times d}$ and $Q^{T}= (x_{1},x_{2},\dots,x_{n})$ with $x_{i} \in \mathbb{R}^{d}$ for $1 \leq i \leq n$, indicating that $Q$ is an $n$-length input sequence with feature dimension $d$. The single-head attention is defined as 
$$
 \operatorname{SoftmaxAttn}(x_{i})=\sum_{j=1}^{n} \frac{\exp\left(\frac{{\langle W_{q}x_{i},W_{k}x_{j}  \rangle_{}}}{\sqrt{ d_{in} }}\right)}{\sum_{j'=1}^{n} \exp\left(\frac{{\langle W_{q}x_{i},W_{k}x_{j'}  \rangle_{}}}{\sqrt{ d_{in} }} \right)}(W_{v}x_{j})
$$ 
for each row vector $x_{i}^{T}$ in the input sequence $Q$, where $W_{k} \in \mathbb{R}^{d_{in}\times d},W_{q} \in \mathbb{R}^{d_{in}\times d},W_{v} \in \mathbb{R}^{d \times d}$ are parameter matrices. Then the output of a self-attention layer is defined as 
\begin{equation}
\operatorname{Softmax} \operatorname{Attn}(Q)=\left[\begin{array}{c}
\operatorname{Softmax} \operatorname{Attn}^T\left(x_1\right) \\
\vdots \\
\operatorname{Softmax} \operatorname{Attn}^T\left(x_n\right)
\end{array}\right] \in \mathbb{R}^{n \times d}
\end{equation}
and it is followed by a position-wise FNN in the form of 
$$
 \operatorname{FNN}(V)= \begin{bmatrix}
\sigma (v_{1}^{T}W_{1}+b_{1})W_{2}+b_{2} \\
\vdots \\
\sigma(v_{n}^{T}W_{1}+b_{1})W_{2}+b_{2}
\end{bmatrix} \in \mathbb{R}^{n \times d} \ \ \\ \textrm{for } V=\begin{bmatrix}
v_{1}^{T} \\
\vdots \\
v_{n}^{T}
\end{bmatrix} \in \mathbb{R}^{n\times d}
$$
where $W_{1} \in \mathbb{R}^{d \times d_{in}'}, W_{2} \in \mathbb{R}^{d_{in}'\times d}$ are connection matrices and $b_{1} \in \mathbb{R}^{ d_{in}'}, b_{2} \in \mathbb{R}^{d}$ are bias vectors.
So an encoder of the Transformer can be expressed as 
\[
 \operatorname{Encoder}(x_{i}) =  \sigma \left(\left( \frac{1}{n}\sum_{j=1}^{n} k_{attn}(x_{i},x_{j})(x_{j}^{T}W_{v})   \right)W_{1} +b_{1} \right)W_{2}+b_{2} \in \mathbb{R}^{1 \times d}
\]
where 
\begin{equation} \label{eq:k-attn}
    k_{attn}(x_i,x_j):= \frac{\exp\left(\frac{{\langle W_{q}x_{i},W_{k}x_{j}  \rangle_{}}}{\sqrt{ d_{in} }}\right)}{\sum_{j'=1}^{n} \exp\left(\frac{{\langle W_{q}x_{i},W_{k}x_{j'}  \rangle_{}}}{\sqrt{ d_{in} }} \right)}. \\
\end{equation}

  Hence the output of a Transformer encoder is 
 $$
\operatorname{Encoder}(Q) =  \begin{bmatrix}
\sigma \left(\left( \frac{1}{n} \sum_{j=1}^{n} k_{attn}(x_{1},x_{j})(x_{j}^{T}W_{v})   \right)W_{1} +b_{1} \right)W_{2}+b_{2} \\
\vdots \\
\sigma \left(\left(  \frac{1}{n}\sum_{j=1}^{n} k_{attn}(x_{n},x_{j})(x_{j}^{T}W_{v})   \right)W_{1} +b_{1} \right)W_{2}+b_{2}
\end{bmatrix} \in \mathbb{R}^{n \times d} \ \  \\
$$ for $Q^T=\begin{bmatrix}
x_{1} ,\cdots,x_{n}\end{bmatrix} \in \mathbb{R}^{d\times n}.$ 

\vspace{1em}

By examining each row of the output, it becomes clear that the encoder of Transformers has the structure of a fully connected layer following a self-attention operation. The self-attention operation is a weighted sum of the input features $\{W_{v}x_{j}\}_{j=1}^n$, where the weights are determined by a specific kernel function. Therefore, given a feature mapping $f:\mathbb{R}^d \to \mathbb{R}$ for the input sequence $Q \in \mathbb{R}^{n \times d}$ and a kernel function $k:\mathbb{R}^d \times \mathbb{R}^d \to \mathbb{R}$, the self-attention can be written as 
\[ \operatorname{k-Attn}(x_{i}) = \frac{1}{n}\sum_{j=1}^{n} k(x_{i},x_{j}) f(x_{j}). \]
In the original Transformer, $k(x_i,x_j)=k_{attn}(x_i,x_j)$ and $f(x)=W'_{v}x$ with $W_{v}' \in \mathbb{R}^d$.
Moreover, with the empirical distribution $\hat{\mu}^{n}=\frac{1}{n} \sum_{j=1}^{n} \delta _{x_{j}}$, we have
\[{\operatorname{k-Attn}(x_{i})} = \int k(x_{i},x) f(x) \, d \hat{\mu}^{n}. \] Then it follows that $\operatorname{k-Attn}(\cdot)$ is an empirical form of $\int k(\cdot,x) f(x)\, d\mu$ where $\mu$ is the probability distribution that $\{ x_{j} \}_{j=1}^{n}$ are drawn from. Consider $Q$ as a realization of the distribution $\mu$, i.e., a collection of samples drawn from $\mu$, and then, we define the attention operator on $P(\Omega)$ as 

\begin{equation} \label{eq:attn}
    \operatorname{attn}(\mu) = \int_{\Omega}  k(\cdot,x)f(x) \, d\mu  
\end{equation}
with some conditions on $k$ and $f$, specified in the subsequent section. In contrast to $\operatorname{SoftmaxAttn}(Q) \in \mathbb{R}^{n \times d}$ for $Q \in \mathbb{R}^{n \times d}$, the operator $\operatorname{attn}$ maps a probability measure to a function. However, a form similar to $\operatorname{SoftmaxAttn}(Q)$ can be obtained through a discretization of the function $\operatorname{attn}(\mu)$. 
To make the attention operator well-defined, in the next subsection we will introduce a metric on $P(\Omega)$, conditions on kernel $k$ and feature mapping $f$, and the form of functionals to be learned.

\subsection{Metric space of $P(\Omega)$ and network structure}
The choice of distances between probability measures is fundamental and has found many applications in deep learning. Many well-known distances share the following similar form. Let $\Omega$ be a compact subset of $\mathbb{R}^{d}$ and $P(\Omega)$ be the set of all Borel probability measures on $\Omega$. For $\mathcal{P},\mathcal{Q} \in P(\Omega)$, the distance $\gamma _{\mathcal{F}}$ between $\mathcal{P}$ and $\mathcal{Q}$ is defined as
$$
 \gamma _{\mathcal{F}}(\mathcal{P},\mathcal{Q}) = \sup _{g \in \mathcal{F}} \left| \int  _{\Omega} g \, d\mathcal{P} -\int  _{\Omega} g \, d \mathcal{Q}  \right| $$
 where $\mathcal{F}$ is a class of real-valued bounded measurable functions on $\Omega$. One can easily observe that $\gamma _{\mathcal{F}}$ satisfies all the conditions for a metric, except for one that $\mathcal{P}=\mathcal{Q}$ if $\gamma _{\mathcal{F}}(\mathcal{P}, \mathcal{Q})=0$. 
But with an appropriate choice of $\mathcal{F}$, $\gamma _{\mathcal{F}}$ can be made a metric on $P(\Omega)$, for example, $\mathcal{F}=C(\Omega)$ the space of continuous functions on $\Omega$, or $\{ g: \left\| g \right\|_{\infty}\leq 1 \}$ where $\left\| g \right\|_{\infty} = \sup _{x \in \Omega} \left| g(x) \right|$.  
It's particularly worth mentioning that the Wasserstein distance $W_{1}$ considered in \citep{yu2023deep,shi2023learning} is also an instance when $\mathcal{F}=\{ g: \left| g \right|_{C^{0,1}}  \leq 1\}$ with the Lipschitz semi-norm $|g|_{C^{0,1}}:=\sup _{x \neq y \ \in \Omega}  \frac{{\left| g(x)-g(y) \right|}}{\left\| x-y \right\|_{2}}$. 
\vspace{0.5em}

In this work, we consider $\mathcal{F}$ to be the unit ball of a reproducing kernel Hilbert space $\mathcal{H}_{k}$ with a reproducing kernel $k$ on $\Omega \times \Omega$, that is, $\mathcal{F}= \{ f \in \mathcal{H}_k:\left\| f \right\|_{\mathcal{H}_{k}} \leq 1 \}$, and denote $\gamma _{k}:=\gamma _{\{ f \in \mathcal{H}_k: \left\| f \right\|_{\mathcal{H}_{k}} \leq 1 \}}$. 
Throughout the paper, we always assume that $k$ is a Mercer kernel, that is, a symmetric, continuous and positive semi-definite kernel, on $\Omega \times \Omega$. 
Because $\int_{\Omega}\sqrt{k(x,x)} \,  d \mathcal{P'}(x) <\infty$ for any $P' \in P(\Omega)$, for any $\mathcal{P},\mathcal{Q} \in P(\Omega)$, it can be inferred from \citep{gretton2006kernel} that 
$$
 \gamma _{k}(\mathcal{P},\mathcal{Q})=\left\| \int  _{\Omega} k(x,\cdot) \, d\mathcal{P}(x) - \int  k(x, \cdot) \, d\mathcal{Q}(x)  \right\|_{\mathcal{H}_{k}}:= \left\| k_{\mathbf{P}}(\mathcal{P}) - k_{\mathbf{P}}(\mathcal{Q})\right\|_{\mathcal{H}_{k}}   $$
 where $k_{\mathbf{P}}(\mathcal{P}):= \int  _{\Omega} k(x, \cdot) \, d\mathcal{P}(x)$. Yet, with an arbitrary Mercer kernel $k$, $\gamma _{k}$ is not always  a metric on $P(\Omega)$, in other words, not always satisfying the condition that $\mathcal{P}=\mathcal{Q}$ if $\gamma _{k}(\mathcal{P},\mathcal{Q})=0$. Many studies have explored conditions on $k$ under which $\gamma _{k}$ becomes a metric on $P(\Omega)$. Here, we just present some conditions useful for the analysis later.

\begin{definition}
    Let $k$ be a Mercer kernel on $\Omega \times \Omega$ where $\Omega$ is a compact subset of $\mathbb{R}^{d}$.  
\begin{itemize}
    \item $k$ is said to be universal if $\mathcal{H}_{k}$ is dense in $C(\Omega)$.
    \item $k$ is said to integrally strictly positive definite if $\int _{\Omega}  \int _{\Omega} k(x,y) \, d\mu (x) \, d \mu(y)>0$ for all non-zero signed finite Borel measures $\mu$ defined on $\Omega$.
\end{itemize}
\end{definition}

In fact, the two definitions above are shown \citep{sriperumbudur2011universality} to be equivalent, but the second form of double integrals can be more useful in the proof later. It's also shown in \citep{gretton2006kernel} that for universal kernels, $k_{\mathbf{P}}$ defines an injective mapping from $P(\Omega)$ to $\mathcal{H}_{k}$, which implies that $\gamma _{k}$ is a metric on $P(\Omega)$. Because a lot of popular kernels in the application are universal, including Gaussian kernels, Laplacian kernels, inverse multiquadrics, Matérn kernels, it's natural to consider the attention operators induced by universal kernels. Now, we give the formal definition of our attention operator:

\begin{definition}
     Let $\Omega \subset \mathbb{R}^{d}$ be compact and $P(\Omega)$ be the set of all Borel probability measures defined on $\Omega$. Suppose that $k$ is a universal kernel on $\Omega \times \Omega$ and $f: \Omega \to \mathbb{R}$ is a continuous function with $c_{f} \leq | f(x) | \leq C_{f}$ for all $x \in \Omega$, where $c_{f},C_{f}>0$ are two constants. Then the attention operator $\operatorname{attn}: (P(\Omega), \gamma _{k}) \to (\mathcal{H}_{k}, \|\cdot\|_{\mathcal{H}_{k}})$ induced by $k$ and $f$ is defined as 
$$
 \operatorname{attn}(\mathcal{P}) = \int  _{\Omega} k(x,\cdot) f(x)\, d\mathcal{P}.$$
\end{definition}

 It is easy to see that for $\mathcal{P}\in (P(\Omega), \gamma _{k})$, 
 $$\left\| \operatorname{attn}(\mathcal{P}) \right\|_{\mathcal{H}_{k}} = \left(\int _{\Omega} \int  _{\Omega} k(x,y)f(x)f(y) \, d\mathcal{P}(x) \, d\mathcal{P}(y) \right)^{\frac{1}{2}} \leq r$$ with $r:=C_{f}\left\| k \right\|_{\infty}$, where $\left\| k \right\|_{\infty} := \sup _{x \in \Omega} \sqrt{ k(x,x) }$. Let $\mathcal{G}_{k,f}$ denote the image of $\operatorname{attn}$ in $\mathcal{H}_{k}$. Then $\mathcal{G}_{k,f}$ is contained in the closed ball $B_{r}:= \{ g \in \mathcal{H}_{k}: \left\| g \right\|_{\mathcal{H}_{k}} \leq r \}$.
The attention operator can be viewed as an embedding from distributions to function representations. The following theorem shows some nice properties of this distribution embedding.  

\begin{theorem}
    Let $k',k$ be two universal kernels defined on $\Omega \times \Omega$, and $f$ defined in Definition 2. Then the attention operator $\operatorname{attn}$ induced by $k$ and $f$ is an injective and continuous mapping from $(P(\Omega), \gamma _{k'})$ to $(\mathcal{H}_{k}, \|\cdot\|_{\mathcal{H}_{k}})$.
\end{theorem}

The kernel $k'$ that metrizes $P(\Omega)$ can be a different universal kernel from the one inducing the attention operator. It has little effect on the properties of the attention operator. However, if we take $k'(x,y)$ to be $f(x)k(x,y)f(y)$ mentioned in the proof of Theorem 1, the attention operator is an isometry between $(P(\Omega), \gamma _{k'})$ and $(\mathcal{G}_{k,f}, \|\cdot\|_{\mathcal{H}_{k}})$, which shows that the attention operator can represent an embedding without any loss of information. But for simplicity of notations, we take both kernels to be the same universal kernel.  

Next, we define the Transformers based on our novel attention operator. As mentioned in the Introduction, $\operatorname{SoftmaxAttn}(Q)$ can be regarded as a discretization of the function $\operatorname{attn}(\mathcal{P})$. Here, for any set of distinct points $\mathbf{T}=\{ t_{1},\dots, t_{|\mathbf{T}|} \} \subset \Omega$, we introduce a sampling operator $[\cdot]_{\mathbf{ T}}: \mathcal{H}_{k} \to \mathbb{R}^{|\mathbf{T}|}$ to discretize $\operatorname{attn}(\mathcal{P})$ where $|\mathbf{T}|$ denotes the size of the set $\mathbf{T}$, such that for any $g \in \mathcal{H}_{k}$, $[g]_{\mathbf{T}}=[g(t_{j})]_{1 \leq j \leq |\mathbf{T}|} \in \mathbb{R}^{|\mathbf{T}|}$. Then we give the following rigorous definition of our Transformer encoder.

\begin{definition}
    Let $\Omega$ be a compact subset of $\mathbb{R}^{d}$, $k$ be a universal kernel, and $\operatorname{attn}$ be the attention operator defined in Definition 2. With distribution inputs from $(P(\Omega), \gamma _{k})$, the Transformer encoder $H_{n_{1},n_{2}}$ of type $(n_{1},n_{2})$ is defined by:
\begin{equation} \label{eq:tran}
     H_{n_{1},n_{2}}(\mathcal{P})= c^{T} \sigma (A[\operatorname{attn}(\mathcal{P})]_{\mathbf{T}}+b)+b_{0} 
\end{equation}
 where $\mathbf{ T}$ is a set of $n_{2}$ points in $\Omega$, $A \in \mathbb{R}^{n_{1}\times n_{2}}$ a parameter matrix, $b \in \mathbb{R}^{n_{1}}$ a bias vector, $c \in \mathbb{R}^{n_{1}}$, $b_{0} \in \mathbb{R}$ and $\sigma$ is ReLU activation function given by $\sigma(x)=\max \{ 0,x \}$.
\end{definition}

\begin{remark}
With the definitions of the attention operator and the Transformer encoder above, there are some points we wish to clarify. First, the boundedness from below condition of  the continuous feature mapping $f$ can be removed in some applications, where a feature mapping is usually learned by a neural network. For example, $f(x)=\sum^{n}_{j}c_{j}(\sigma(a_{j} \cdot x+b_{j})+\epsilon _{j})$ with all $\epsilon_j >0$, then we have 
$$
\int_{\Omega} k(x, \cdot) f(x) \, d\mathcal{P} = \sum_{j=1}^{n}c_{j}\int_{\Omega} k(x, \cdot) (\sigma(a_{j} \cdot x+b_{j})+\epsilon _{j}) \, d\mathcal{P}
$$
and with the universality of shallow nets, there's no loss of the expressivity. Together with Theorem 1, we can observe that the composition structure of attention layers and FNN layers is crucial to the success of Transformers in embedding probability measures into function representations and learning diverse feature representations, which is also consistent with the Transformer's powerful data compression capabilities in practical applications.
\end{remark}

\begin{remark} \label{re:sa}
Note that when the input is an empirical distribution $\hat{\mu}^n = \frac{1}{n}\sum_{j=1}^{n}\delta_{x_j}$, let $\mathbf{T}'= \{x_j\}_{j=1}^n$ and then $[\operatorname{attn}(\hat{\mu}^n)]_{\mathbf{T}'}$ is exactly the self-attention. For a functional $H_{n_{1},n_{2}}$ defined on $(P(\Omega), \gamma _{k})$, generally speaking, $n_{2}$ controls the degree of discretization and $n_{1}$ controls the accuracy of approximation to the target functional. With a larger $n_{2}$, there is less information loss from the original distribution. In the application, this may also explain why engineers always manage to increase the length of the input sequence (i.e., the number of tokens) for the self-attention module in LLMs.
\end{remark}

\section{Main Results}

This section provides the main results on learning capabilities of Transformer-based networks in distribution regression. First we define a functional class induced by Barron functionals and an example to illustrate its powerful expressivity, and then demonstrate the approximation rate of the defined functional class by a Transformer encoder. With the approximation rate and an estimation of covering numbers, we obtain a generalization bound by an oracle inequality in the final subsection. 

The attention operator embeds each Borel probability measure into a function in the closed ball $B_{r}$ in $\mathcal{H}_{k}$. Then within the framework of distribution regression, we require a functional that maps functions in $B_{r}$ to $\mathbb{R}$. Here, we consider a functional class produced by Barron functionals \citep{barron1993universal} and the definition is given below.

For a real-valued functional $\Phi$ on a Hilbert space $(\mathcal{H}, \|\cdot\|)$, we say that $\Phi$ is represented by a Fourier distribution $\tilde{F}$ on some domain $A \subset \mathcal{H}$ where $\tilde{F}$ is a complex-valued measure $\tilde{F}(d \omega)=e^{i \theta(\omega)}F(d \omega)$ if $\Phi(g)=\int_{\mathcal{H}}e^{i \langle g,\omega \rangle}\tilde{F}(d \omega)$ for all $g \in A$. Here $F(d \omega)$ denotes that magnitude distribution and $\theta(\omega)$ denotes the phase at the frequency $\omega$.

\begin{definition}
    For each $r,C >0$, let $\Gamma_{r,C}(\mathcal{H})$ be the set of functionals $\Phi$ on $B_{r}:=\{ g \in \mathcal{H}: \|g\| \leq r  \}$ such that there's a Fourier distribution $\tilde{F}$ representing $\Phi$ on $B_{r}$ satisfying $\int_{\mathcal{H}}  \|\omega\|F (d \omega) \leq  C$. Every functional in $\Gamma _{r,C}(\mathcal{H})$ is called a Barron functional.
\end{definition}

We shall assume that the target function (regression function) in distribution regression has the form
$$
\Phi(\operatorname{attn}(\mathcal{P}))=\Phi\left( \int  _{\Omega }k(x, \cdot)f(x) \, d\mathcal{P} \right) \textrm{ where }\Phi \in \Gamma _{r,C}(\mathcal{H}_{k}).$$
To demonstrate the powerful expressivity of the Barron functional class, we provide an example here.

\begin{example}
Consider the ridge functional $\Phi(x)=g(\langle  a,x \rangle_{\mathcal{H}_{k}})$ with $\| a \|_{\mathcal{H}_{k}}=1$ and some univariate continuous function $g$ with Fourier transform $\hat{g}$ on $\mathbb{R}$ satisfying $\int_{\mathbb{R}}|t||\hat{g}(t)| \, dt < \infty$. Then $\Phi(x) = \int  e ^{it\langle a,x \rangle_{\mathcal{H}_{k}}} \hat{g}(t)\, dt$, which means that a Fourier distribution $\tilde{F}$ supported on the set of $\{ ta:t\in \mathbb{R} \}$, represents $\Phi$ on $B_{{r}}$ for any ${r}>0$. $\Phi$ is a Barron functional on $B_{{r}}$, as long as $g$ is a Barron functional on $\mathbb{R}$. It is shown in \citep{barron1993universal} that $g$ is a Barron functional on $[-{r}, {r}]$ when the second derivative of $g$ is continuous. In this case, $\Phi$ is a Barron functional. When the input space is the embedding of all Borel probability measures and $\mathcal{H}_{k}$ is the corresponding RKHS, we obtain that 
\begin{align}
    \Phi(\operatorname{attn}({\mathcal{P}}))&=g(\langle  a,\operatorname{attn}({\mathcal{P}}) \rangle_{\mathcal{H}_{k}}) \notag\\ &=g\left( \left\langle  a,\int  k(x,\cdot) f(x)\, d\mathcal{P}   \right\rangle_{\mathcal{H}_{k}} \right)\notag\\&=g\left( \int  a(x) f(x)\, d\mathcal{P}  \right) \label{eq:int} .
\end{align}
Since the definitions of many statistics (e.g., moments) are closely related with integration w.r.t. probability measures and $\mathcal{H}_{k}$ is dense in $C(\Omega)$, (\ref{eq:int}) is a nice tool to capture the relation between statistics of distributions and respond variables.

Beyond statistics of distributions, we may also consider a multivariate case and have distribution projections of the feature random variables $f(X), X \thicksim \mathcal{P}$ to retain as much information as desired. 
Note that the above case can be extended to ridge functionals with multiple features, i.e., 
$$
 \Phi(x) = g(\langle a_{1},x \rangle_{\mathcal{H}_{k}},\dots,\langle  a_{d'},x \rangle_{\mathcal{H}_{k}}) $$
 with $\left\| a_{j} \right\|_{\mathcal{H}_{k}}=1$ for all $1 \leq j \leq d'$ and some continuous function $g$ defined on $\mathbb{R}^{d'}$. Similarly, if the partial derivatives of $g$ of order $\lfloor d'/2 \rfloor+2$ are continuous in $\mathbb{R}^{d'}$, then $g$ is a Barron functional on $[-{r},{ r}]^{d'}$, which implies that $\Phi$ is also a Barron functional.
This form of feature embedding with the attention operator is much more flexible than the functions with polynomial features considered in \citep{shi2023learning,yu2023deep}. We don't need to design a specific network structure for feature functions, but can still learn feature functions from a dense subset of $C(\Omega)$. Moreover, the above ridge functional form is just one case of the class of Barron functionals. 
\end{example}

\subsection{Approximation rate of Barron functional}
We establish an approximation theory for a class of functionals $\Phi _{k,f}$ defined on $(P(\Omega), \gamma _{k})$ by exploiting the proposed Transformer encoder. The functional $\Phi _{k,f}$ has the composition form of a Barron functional and the attention operator:
\begin{align} \label{eq:phi}
   \Phi _{k,f}(\mathcal{P}) := \Phi(\operatorname{attn}(\mathcal{P})) =\Phi\left( \int_{\Omega} k(x,\cdot)f(x) \, d\mathcal{P}  \right)
\end{align}
with $\Phi \in \Gamma _{{r,C}}(\mathcal{H}_{k})$ and $r=C_{f}\left\| k \right\|_{\infty}$. 
\vspace{1em}

In \citep{barron1993universal}, we have the following approximation lemma for Barron functionals defined on a Hilbert space $\mathcal{H}$. A sigmoidal function $\phi$ on $\mathbb{R}$ means a bounded measurable function satisfying $\lim_{t \to \infty} \phi(t)=1$ and $\lim_{t \to -\infty}\phi(t)=0$.
\begin{lemma}

For ${r},C>0$, $\Phi \in \Gamma _{{r},C}(\mathcal{H}_{k})$, $\beta>0, n \in \mathbb{N}$, sigmoidal function $\phi$ on $\mathbb{R}$, probability measure $\nu$ on $B_{{r}}$, there is a function $\Psi_{n}(g)=\sum_{p=1}^{n}c_{p}\phi( \langle  a_{p}, g \rangle_{\mathcal{H}_{k}} +b_{p})+\Phi(0)$ with $\left\| a_{p} \right\|_{\mathcal{H}_{k}}\leq \frac{\beta}{{r}}, \left| b_{p} \right|\leq \beta$ and $\|c\|_{1} \leq 2{r}C$, such that 
$$
\int_{B_{{r}}} ({\Phi}(x)-\Psi_{n}(x))^{2} \,\mu(dg) \leq {(2{r}C)^{2}}\left( \frac{1}{n^{1/2}} +\eta _{\beta} \right)^{2}$$ where $\eta_{\beta}=\inf_{0<\epsilon \leq \frac{1}{2}} \{ 2 \epsilon + \sup _{|z|\geq \epsilon} \left| \phi(\beta z)-1_{\{ z >0 \}} \right| \}$.

\end{lemma}

However, the approximation form $\Psi _{n}$ cannot be directly applied with popular network structures, because it involves functions $\{a_{p} \}  \subset  \mathcal{H}_{k}$ as parameters. One idea is to utilize $\sum_{q} a_{p,q}k(t_{q},\cdot)$ with $a_{p,q} \in \mathbb{R}$ and $t_{q} \in \Omega$ to approximate each $a_{p}$, and then, by kernel trick, the inner product $\langle  a_{p} ,g \rangle_{\mathcal{H}_{k}}$ can be approximated by a linear combination the function values $\{ g(t_{q}) \}$ at the sampled points, which exactly matches the form of our Transformer encoder. This idea was studied recently in \citep{zhou2024approximation}. Then by using the Transformer encoder with a Gaussian kernel $k=\exp (- {\|x-y\|^{2}}/\alpha^{2})$ ($\alpha>0$), we have the following result on $L^{2}_{\rho _{\mathcal{U}}}$ approximation rates of the functional $\Phi _{k,f}$ with the same Gaussian kernel, where we denote $\rho _{\mathcal{U}}$ as the marginal distribution $\rho$ on $\mathcal{U}=P(\Omega)$ and $(L^{2}_{\rho_{\mathcal{U}}}, \| \cdot \|_{\rho})$ as the space of square integrable functions with respect to $\rho _{\mathcal{U}}$.

\begin{theorem}
Let $k$ be a Gaussian kernel $k(x,y)= \exp \{-\|x-y\|^2/\alpha^2\}$ with $\alpha >0$ and $\Omega = [-1,1]^d$. For every functional $\Phi _{k,f}$ defined by (\ref{eq:phi}) for any $r,C>0, n \in \mathbb{N}$, there exists a Transformer encoder $h_{n}$ in the hypothesis $\mathcal{H}_{R,n}$ such that 
$$
 \left\| \Phi _{k,f} -h_{n}\right\|_{\rho} \leq \frac{(4C_{f}+r)C}{n^{\frac{1}{2}}}.  $$
The hypothesis space $\mathcal{H}_{R,n}$ is a class of functions in Definition 3 of type $(2n, s(n))$ with $s(n):=\lceil C_{k,d}(\log n)^{d} \rceil$ such that 
$$
 \mathcal{H}_{R,n} = \left\{  H_{n,s(n)}: \|c\|_{1} \leq Rn^{\frac{1}{2}}, |A_{p,q}| \leq R(\log n)^{-\frac{d}{2}}n^{R\log n}, \left\| b \right\|_{\infty} \leq R , \ b_{0}= \Phi(0) \right\} $$
where $A=(A_{p,q})$, $C_{k,d}$ is a constant depending on $d$ and kernel $k$, and $R$ depends only on $r,C,d$ and kernel $k$. The total number of free parameters of Transformer encoder is $O(n(\log n)^{d})$.
\end{theorem}

The above theorem gives rates of approximating a class of functionals by a Transformer encoder, with the complexity bound on the total number of free parameters and the parameter bounds on connection matrices and bias vectors. These will be useful later to derive an estimation of covering numbers and generalization bounds in distribution regression.

\subsection{Distribution regression with Transformers}
This section conducts generalization analysis of the empirical risk minimization (ERM) algorithm for distribution regression with Transformer encoders. 
Now we propose the ERM algorithm for two-stage distribution regression. Take $\mathcal{Z}=\mathcal{U}\times \mathcal{Y}$, where $\mathcal{U}=P(\Omega)$ is the input space of all Borel probability measures on $\Omega=[-1,1]^d$ and $\mathcal{Y}=[-M,M]$ is the output space with $M>0$. The regression function $\varphi _{\rho}$ on $\mathcal{U}$ is defined as
$$
 \varphi _{\rho}(\mu) = \int  _{\mathcal{Y}} y \, d \rho(y|\mu ) 
$$
where $\rho(\cdot|\mu )$ is the conditional distribution at $\mu$ induced by $\rho$, and it minimizes the mean squared error for $\varphi:\mathcal{U}\to \mathcal{Y}$,
$$
\mathcal{E}(\varphi) = \int  _{\mathcal{Z}} (\varphi(\mu) - y)^{2} \, d\rho.  $$
For establishing a learning theory for the ERM algorithm, it is crucial: whether $\mathcal{H}_{R,n}$ is well defined as a compact hypothesis space. To answer the question, we denote $C(P(\Omega))$ with norm $\|\varphi\|_{\infty}:= \sup_{\mu \in P(\Omega)}|\varphi(\mu )|$, to be the Banach space of continuous functions on $(P(\Omega), \gamma _{k})$. Then we have the following theorem:

\begin{theorem} \label{thm:cover}
    The hypothesis $\mathcal{H}_{R,n}$ of Transformer encoders is a compact subset of $C(P(\Omega))$.
\end{theorem}

With Theorem \ref{thm:cover}, the covering number $\mathcal{N}(\mathcal{H}_{R,n}, \epsilon, \|\cdot\|_{\infty})$ of $\mathcal{H}_{R,n}$ as a subset of $C(P(\Omega))$ makes sense, where $\mathcal{N}(\mathcal{H}_{R,n}, \epsilon, \|\cdot\|_{\infty})$ denotes the minimum number of balls with radius $\epsilon>0$ whose union covers $\mathcal{H}_{R,n}$ in the space $C(P(\Omega))$. The estimation of the covering number plays a vital role in deriving a generalization bound for distribution regression. The related details will be presented in Section 5.

Recall that for the second stage dataset $\hat{D}=\{ \big(\{ x_{i,j} \}_{j=1}^{m_{2,i}}, y_{i}  \big) \}_{i=1}^{m_{1}}$ in two-stage distribution regression, the empirical target functional from the ERM algorithm with the hypothesis space $\mathcal{H}_{R,n}$ is the functional defined as
$$
\varphi _{\hat{D},R,n}= \arg \min _{\varphi \in \mathcal{H}_{R,n}} \mathcal{E}_{\hat{D}}(\varphi)$$ with $$\mathcal{E}_{\hat{D}}(\varphi):= \frac{1}{m_{1}}\sum_{i=1}^{m_{1}} \big( \varphi(\hat{\mu }_{i}^{m_{2,i}})-y_{i}\big)^{2},$$
where the existence of the minimizer $\varphi _{\hat{D},R,n}$ is guaranteed by the compactness of $\mathcal{H}_{R,n}$. 

We now define the truncation operator $\pi _{M}$ on the space $C(P(\Omega))$ as
\begin{equation*}
 \pi _{M}(\varphi)(\mu) = 
\begin{cases}M, & \text { if } \varphi(\mu)>M, \\ -M, & \text { if } \varphi(\mu)<-M, \\ \varphi(\mu), & \text { if }-M \leq \varphi(\mu) \leq M.\end{cases}
\end{equation*}

Since the regression function $\varphi _{\rho}$ is bounded by $M$, the truncated empirical target functional 
$$
 \pi _{M}\varphi _{\hat{D},R,n} $$
 is considered as the final estimator.

\begin{remark}
    The distribution regression framework considered here involves a two-stage sampling process, which makes our Transformer encoder (\ref{eq:tran}) compatible with practical applications, particularly when dealing with sequential inputs like sentences in NLP. From the two-stage sampling process, $m_{2,i}$ controls the length of the sequential input $i$, which can be understood as the number of tokens in the input sequence $i$ in the case of NLP. The two-stage sampling process allows us to study the generalization capabilities of the model under the practical data structure, while also taking into account the approximation ability to abstract functionals. 
\end{remark}

To derive the excess generalization error $\mathcal{E}(\pi _{M}\varphi _{\hat{D},R,n})-\mathcal{E}(\varphi)$, we introduce the empirical error of the first-stage sample 
$$
 \mathcal{E}_{D}(\varphi):= \frac{1}{m_{1}} \sum^{m_{1}}_{i=1} (\varphi(\mu _{i})-y_{i})^{2} $$ and then we can obtain a decomposition of the excess generalization error in the following lemma which can be easily seen from the fact that $\mathcal{E}_{\hat{D}}(\pi _{M}\varphi _{\hat{D},R,n}) \leq \mathcal{E}_{\hat{D}}(h)$.

\begin{lemma}
    For any $h \in \mathcal{H}_{R,n}$ and $\varphi _{\hat{D},R,n}$ defined in (\ref{eq:var}), we have
 $$
 \begin{aligned}
& \mathcal{E}\left(\pi_M \varphi_{\hat{D}, R, n}\right)-\mathcal{E}\left(\varphi_\rho\right) \leq \mathcal{E}\left(\pi_M \varphi_{\hat{D}, R, n}\right)-\mathcal{E}_D\left(\pi_M \varphi_{\hat{D}, R, n}\right)+\mathcal{E}_D\left(\pi_M \varphi_{\hat{D}, R, n}\right) \\
& -\mathcal{E}_{\hat{D}}\left(\pi_M \varphi_{\hat{D}, R, n}\right)+\mathcal{E}_{\hat{D}}(h)-\mathcal{E}_D(h)+\mathcal{E}_D(h)-\mathcal{E}(h)+\mathcal{E}(h)-\mathcal{E}\left(\varphi_\rho\right)
\end{aligned}  $$
which can be bounded by the summation
$$\mathcal{I}_{1}(D, \mathcal{H}_{R,n})+\mathcal{I}_{2}(D,\mathcal{H}_{R,n})+\left| \mathcal{I}_{3}(\hat{D},\mathcal{H}_{R,n}) \right|+\left| \mathcal{I}_{4} (\hat{D},\mathcal{H}_{R,n})\right|+R(\mathcal{H}_{R,n})$$ in which
$$
\begin{aligned}
& \mathcal{I}_1(D, \mathcal{H}_{R,n})=\left\{\mathcal{E}\left(\pi_M \varphi_{\hat{D}, R, n}\right)-\mathcal{E}\left(\varphi_\rho\right)\right\}-\left\{\mathcal{E}_D\left(\pi_M \varphi_{\hat{D}, R, n}\right)-\mathcal{E}_D\left(\varphi_\rho\right)\right\} \\
& \mathcal{I}_2(D, \mathcal{H}_{R,n})=\left\{\mathcal{E}_D(h)-\mathcal{E}_D\left(\varphi_\rho\right)\right\}-\left\{\mathcal{E}(h)-\mathcal{E}\left(\varphi_\rho\right)\right\} \\
& \mathcal{I}_3(\hat{D}, \mathcal{H}_{R,n})=\mathcal{E}_D\left(\pi_M \varphi_{\hat{D}, R, n}\right)-\mathcal{E}_{\hat{D}}\left(\pi_M \varphi_{\hat{D}, R, n}\right) \\
& \mathcal{I}_4(\hat{D}, \mathcal{H}_{R,n})=\mathcal{E}_{\hat{D}}(h)-\mathcal{E}_D(h), \quad R(\mathcal{H}_{R,n})=\mathcal{E}(h)-\mathcal{E}\left(\varphi_\rho\right) .
\end{aligned} $$

\end{lemma}

Based on the two-stage error decomposition, the two-stage oracle inequality for distribution regression in the hypothesis space $\mathcal{H}_{R,n}$ of our proposed Transformer encoder is established in the following theorem to be proved in Section 5.

\begin{theorem}
Consider the distribution regression framework with the first stage sample size $m_{1} \in \mathbb{N}$, and the second stage sample size $\min \{ m_{2,i}: 1 \leq i \leq m_{1} \}=m_{2}$. Then for $n \geq  3$, any $h \in \mathcal{H}_{R,n}$ and $\epsilon >0$, we have 
$$
 \begin{aligned}
& \operatorname{Prob}\left\{\left\|\pi_M \varphi_{\hat{D}, R, n}-\varphi_\rho\right\|_\rho^2>2\left\|h-\varphi_\rho\right\|_\rho^2+8 \epsilon\right\} \\
\leq & \mathcal{N}\left(\mathcal{H}_{R, n}, \frac{\epsilon}{16 M},\|\cdot\|_{\infty}\right) \exp \left\{-\frac{3 m_{1} \epsilon}{2048 M^2}\right\} \\
+ & \exp \left\{-\frac{m_{1} \epsilon^2}{2\left(3 M+\|h\|_{\infty}\right)^2\left(\left\|h-\varphi_\rho\right\|_\rho^2+\frac{2}{3} \epsilon\right)}\right\}\\
+ &4m_{1}s(n) \exp \left\{  - \frac{m_{2} \epsilon^{2}}{128\max \{ \left\| h \right\|_{\infty}^{2},M^{2}  \}C_{f}^{2}C_{5}^{2}n^{4R\log n}(\log n)^{d}}  \right\}
\end{aligned}  $$
where $C_{5}:=2C_{k,d}R^{2}$ is a constant depending on $d$, $R$ and kernel $k$.
\end{theorem}

Based on the the oracle inequality for distribution regression, the excess generalization error can be bounded in the following theorem to be proved in Section 5, where we assume that the regression function $\varphi _{\rho}$ belongs to the functional class defined in (\ref{eq:phi}).

\begin{theorem}
Suppose that the regression function $\varphi _{\rho}$ has the form (\ref{eq:phi}) with $\Phi(0)=0$. If the total number $N$ of free parameters of Transformer encoders and the second stage sample size $m_{2}$ are chosen by
$$
 N=\left \lfloor\mathcal{A}_{7}m_{1}^{\frac{1}{2}}\left( \log (\mathcal{A}_{3}m_{1}^{\frac{1}{2}}) \right)^{d}\right\rfloor  \textrm{ and }m_{2} = \left\lceil{ \mathcal{A}_{5} \left( \mathcal{A}_{3}m_{1}^{ \frac{1}{2} } \right)^{8R\log\left( \mathcal{A}_{3}m_{1}^{\frac{1}{2}} \right)}}\right\rceil,$$
 then for the truncated estimator produced by the distribution regression framework with Transformer encoders, 
 $$
 \mathbb{E}\{ \mathcal{E}(\pi _{M}\varphi_{\hat{D},R,n})-\mathcal{E}(\varphi _{\rho}) \} \leq \mathcal{A}_{6} m_{1}^{-\frac{1}{2}}\left( \log\left( \mathcal{A}_{3}m_{1}^{\frac{1}{2}} \right) \right)^{d+2} $$
where $\mathcal{A}_{3},\mathcal{A}_{4},\mathcal{A}_{5},\mathcal{A}_{6},\mathcal{A}_{7}$ are constants depending only on $C_f,C,M,d$ and $\alpha$.
\end{theorem}
\begin{remark}
In the proof of Theorem 5, we choose $n$ to scale up with $m_1^{\frac{1}{2}}$ for a balance between the approximation error and the estimation error. Recall that $n$ controls the complexity of the FNN layer: a larger $n$ increases the hypothesis complexity of the Transformer class, leading to better approximation and more diverse feature representations. Besides, to achieve a nice generalization performance, the hypothesis complexity of the FNN layer also scales up with the first stage sample size $m_1$. For the case of NLP, $m_1$ measures the diversity of semantics in the datatset $D$. Especially when training an LLM with a dataset size often exceeding hundreds of terabytes of text data, $m_1$ becomes incredibly large. This poses challenges for efficiently scaling up the complexity of the FNN layer to achieve a balance. We will discuss this point in subsection 4.3 below.    
\end{remark}

\section{Conclusion and Discussion}
In this paper, we propose the attention operator for modeling attention mechanisms in Transformers, and apply the two-stage sampling process to understand the training process of Transformers in practical applications. We analyze the expressivity and demonstrate the generalization capacity of the proposed Transformer structures. In this section, we exploit the established framework and theoretical results to develop deep understanding of various tricks and techniques with Transformers in practical applications. 

\subsection{Kernel normalization}
Recall the original self-attention module (\ref{eq:k-attn}) in \citep{vaswani2017attention}
\begin{equation*}
    k_{a t t n}\left(x_i, x_j\right):=\frac{\exp \left(\frac{\left\langle W_q x_i, W_k x_j\right\rangle}{\sqrt{d_{i n}}}\right)}{\sum_{j^{\prime}=1}^n \exp \left(\frac{\left\langle W_q x_i, W_k x_{j^{\prime}}\right\rangle}{\sqrt{d_{i n}}}\right)} .
\end{equation*}

If we consider the normalization factor $\sum^{n}_{j'=1} \exp  \left( \frac{{\langle W_{q}x_{i},W_{k}x_{j'} \rangle}}{\sqrt{  d_{in} }} \right)$ independently (since it depends on the input sequence $Q$), then its probability embedding can be reformulated as 
\begin{equation*}
    \overline{\operatorname{attn}}_k(\mu)= \int_{\Omega} \frac{k(x,\cdot)}{\int_{\Omega} k(y,\cdot) \, d \mu(y)} f(x) \, d\mu(x), \mu \in P(\Omega),
\end{equation*}
where $k$ is a Laplace or Gaussian kernel. It can be further written as
\begin{equation} \label{eq:fk}
    \overline{\operatorname{attn}}_k(\mu)=  \frac{\int_{\Omega} {k(x,\cdot)} f(x) \, d\mu(x)}{\int_{\Omega} k(y,\cdot) \, d \mu(y)}.
\end{equation}
If we allow multiple features represented by different FNNs (further discussed in subsection 4.4 below), then $\overline{\operatorname{attn}}_k$ can be approximated by our attention operator $\operatorname{attn}$ (\ref{eq:attn}) induced by the same $k$ and $f$. 

Let $\Omega = [-1,1]^d$. We define $f'$ as $f'(y)=f_1(y) + f_1(-y)$ for $y \in [-1,1]$ where
\begin{equation*}
    f_1(y) = \frac{1}{2}\left[  \sigma(y+1)-2\sigma(y-1) +\sigma(y-3)   \right].
\end{equation*}
Then it's easily verified that $f'=1$ on $[-1,1]$. For $x \in [-1,1]^d$, take $f_1(x):=\frac{1}{d}\mathbf{1}^Tf'(x)=1$ in $[-1,1]^d$ where $f'$ apply element-wise on $x$ and $\mathbf{1}=(1,\cdots,1)^T \in \mathbb{R}^d$. In other words, the normalization function can be represented by our attention operator with the feature function $f'$ constructed by the above FNN.

 Note that $\overline{\mathrm{attn}}_{k}$ can be written as a composition of $g_1\left(\mathrm{attn}_{k,f}(\mu), g_2(\mathrm{attn}_{k,f'}(\mu))\right)$, where $\mathrm{attn}_{k,f}$ denotes the attention operator induced by the kernel $k$ and feature $f$, $g_1(y_1,y_2):=y_1 y_2$ and $g_2(y_3):=1/y_{3}$, and $k$ is a Laplace or Gaussian kernel, which implies $0<\mathrm{attn}_{k,f}(\mu) \le r$ and $0<c'<\mathrm{attn}_{k,f'}(\mu) \le \|k\|_{\infty}^2$. Then $g_1,g_2$ can both be approximated by fully connected neural networks \citep{yarotsky2017error}. In conclusion, we show that $\overline{\mathrm{attn}}_k$ can be approximated by our attention operator $\mathrm{attn}$ with multiple features.

In our framework, we are mainly concerned with the function space and the properties of the attention operator induced by $k,f$. Therefore, a normalization function is of less importance in our analysis, since it can be separated from the attention operator as in (\ref{eq:fk}). However, in practical applications, a normalization function has two advantages. First, it introduces asymmetric dependency as an inductive bias into the attention mechanism, which means that $k_{\mathrm{attn}}(x_i,x_j)$ may not be equal to $k_{\mathrm{attn}}(x_j,x_i)$. The inductive bias is very useful for modeling certain data structures such as natural languages, because words often have hierarchical and directional relationships. Second, a normalization function makes the sum of kernel weights always equal to $1$, which can be useful for auto-regressive tasks \citep{radford2018improving,devlin2019bert,raffel2020exploring}. By normalizing the weights, the function effectively balances the contribution of each component and enhances the stability and consistency of the model.

\subsection{Discretization subsets T}

Distribution regression is a special case of operator learning \citep{song2023approximation1,song2023approximation2}. To construct a computable Transformer encoder, we apply two techniques: the two-stage sampling process and the kernel discretization w.r.t. $\mathbf{T}$. The elements in $\mathbf{T}$ are often called queries in NLP. In our proof, we choose the set of queries to be the uniform mesh on $\Omega$. Then for any function in $\mathcal{H}_k$, we can apply the kernel trick to replace function parameters by function values at each query. Although it is a universal choice to recover information from any function parameters in $\mathcal{H}_k$, the uniform mesh also introduces the term $(\log n)^d$ that still rules out practical use of the framework in high-dimensional cases. 

However, there are several cases in practice showing that with an appropriately chosen query set of (much) smaller size, Transformer encoders still perform well on various tasks. The most common is self-attention, as mentioned in Remark \ref{re:sa}. The choice of query sets is adaptive to each input in self-attention, and more precisely, a set of second-stage samples (independent of the dimension). It would be interesting to investigate the underlying mechanism of self-attention within our framework. Another example is that the query set $\mathbf{T}$ can be learned from training data using stochastic gradient descent, which is usually applied in prompt tuning \citep{liu2023pre} and Q-former for cross-modal alignment \citep{li2023blip}. The basic idea behind these techniques can be understood with our theoretical framework: When handling specific learning tasks or aiming to compress data for more refined feature representations (e.g., low-dimensional features), we often do not require a high-resolution uniform mesh as a universal choice to retain as much information as possible. Instead, we just need a set of queries, with a significantly smaller size, which performs well for a particular task or certain data structure. These queries can be obtained by optimizing the corresponding loss functions.

\subsection{FNN for learning features}

Note that the attention operator primarily facilitates the efficient compression of probability distributions, whereas the FNN learns the pertinent features of the Barron functional (shown in Theorem 2). This theoretical discovery also aligns with practical engineering experiences in fine-tuning LLMs, such as adapter modules \citep{houlsby2019parameter}. Fine-tuning is often required to adapt the existing models to new tasks or datasets \citep{hetowards}. When applied with pretrained LLMs, engineers always fix the parameters of the original network and add only some FNN modules with a few trainable parameters into the existing model, e.g., adapters \citep{houlsby2019parameter}. In this way, the training for adaptation to new tasks or datasets can be dramatically decreased for pretrained LLMs with billions of parameters, and the well-trained small modules can be directly plugged into other models to transfer features learned for new tasks. In our theory, features for target functional are entirely learned by the FNN layers and the attention operator merely embeds probability measures into function representations. In other words, for different target functionals, all trainable parameters are contained within the FNN layer, while the attention operator retains no trainable parameters. This provides a theoretical foundation for the successful application of adapters.

Another topic arising from our generalization analysis (Theorem 5) of distribution regression with Transformer encoders, is how to efficiently scale up the complexity of the FNN layer. Recall that in order to balance the approximation error and estimation error, we take $n=O(m_1^{\frac{1}{2}})$ where $m_1$ is the number of first-stage samples (i.e., probability measures in $P(\Omega)$). In our interpretation of NLP, $m_1$ quantifies the complexity of semantics within a dataset. When training LLMs with increasingly large corpora of texts, $m_1$ becomes extremely large. Then to achieve better generalization performance, we need to scale up the complexity of the FNN layer dramatically. However, this poses challenges in training process. A popular solution is called Mixture of Experts (MoE) \citep{eigen2013learning,shazeer2016outrageously,jiang2024mixtral}. Roughly speaking, the basic idea is that for each query, we may choose one out of a pool of FNNs. A gating network decides which FNN to activate based on the input query. This enables efficient scaling with Transformers, since we can increase the number of FNNs in the pool and just activate one FNN each time in the training. It would also be interesting to investigate the MoE mechanism further within the distribution regression framework.

\section{Proof of Main Results}

\subsection{Proof of Theorem 1}
\begin{proof}
    Recall
$$
 \operatorname{attn}(\mu) = \int  _{\Omega} k(x, \cdot)f(x) \, d\mu  $$
for $\mu \in (P(\Omega), \gamma _{k'})$ where $\gamma _{k'}$ is a metric on $P(\Omega)$ induced by the universal kernel $k'$. 

For $\mathcal{P},\mathcal{Q} \in (P(\Omega), \gamma _{k'})$, we have 
 
 \begin{align}
&\left\| \operatorname{attn}(\mathcal{P})-\operatorname{attn}(\mathcal{Q}) \right\|_{\mathcal{H}_{k}}^2\notag\\
=&\left\|  \int_{\Omega} k(x,\cdot) f(x)  \, d\mathcal{P} - \int _{\Omega} k(x,\cdot) f(x)\, d\mathcal{Q}   \right\|_{\mathcal{H}_{k}}^2 \notag \\
  =&\int  _{\Omega}  \int_{\Omega} k(x, y) f(x) f(y)\, d\mathcal{P}(x) d\mathcal{P}(y) \ \notag \\
 & \quad + \int_{\Omega} \int _{\Omega}  k(x, y) f(x) f(y) \, d\mathcal{Q}(x) d \mathcal{Q}(y) \notag \\
	& \quad -2 \int  _{\Omega} \int  _{\Omega} k(x,y)f(x)f(y)\, d\mathcal{P}(x)  \, d\mathcal{Q}(y) \notag\\
 =& \left\| \int_{\Omega} f(\cdot)k(x, \cdot) f(x) \, d\mathcal{P} - \int_{\Omega} f(\cdot)k(x, \cdot) f(x) \, d\mathcal{Q} \right\|_{\mathcal{H}_{\tilde{k}}}^2   \notag\\
=&\left\| \int_{\Omega} \tilde{ k}(x, \cdot)  \, d\mathcal{P} -\int_{\Omega} \tilde{ k}(x, \cdot)  \, d\mathcal{Q}\right\|_{\mathcal{H}_{\tilde{k}}} = \left\| \tilde{k}_{\mathbf{P}}(\mathcal{P}) - \tilde{k}_{\mathbf{P}}(\mathcal{Q}) \right\|_{\mathcal{H}_{\tilde{k}}} \label{attn},
\end{align} 
where $\tilde{k}(x,y):=f(x)k(x,y)f(y)$. It's easy to see that $\tilde{k}$ is a Mercer kernel on $\Omega \times \Omega$. For any finite nonzero signed Borel measure $\mathcal{P}$, define $\mathcal{P} _{f}(E) = \int  _{E} f\, d\mathcal{P}$ for any Borel set $E$. Then $\mathcal{P} _{f}$ is also a finite nonzero signed Borel measure, since $f$ is measurable and $0 < c_{f} \leq |f(x)|\leq C_{f}$ for all $x \in \Omega$. It follows by the universality of Mercer kernel $k$ that

\begin{align}
\int  _{\Omega} \int  _{\Omega} \tilde{k}(x,y) \, d\mathcal{P}(x)  \, d\mathcal{P}(y) &= \int  _{\Omega} \int  _{\Omega} k(x,y) f(x)d\mathcal{P}(x) f(y) d\mathcal{P}(y)  \\
&=\int  _{\Omega} \int  _{\Omega} k(x,y) d\mathcal{P} _{f}(x)d\mathcal{P} _{f}(y) >0
\end{align}
which implies that $\tilde{k}$ is an integrally strictly pd kernel. Therefore, $\gamma _{\tilde{k}}$ is also a metric on $P(\Omega)$ and then $\operatorname{attn}$ is an injective mapping defined on $P(\Omega)$.

Any $g \in C(\Omega)$ can be approximated by $\sum_{p=1}^{N}\alpha _{p} \tilde{k}(\cdot,t_{p})$ to an arbitrary accuracy when $N$ is large enough, because 
$$ 
g(x)-\sum_{p=1}^{N}\alpha _{p}f(x)k(x,t_{p})f(t_{p})=f(x)\left( \frac{g(x)}{f(x)} -\sum_{p=1}^{N}\tilde{\alpha}_{p}k(x,t_{p})\right)
$$
with $\alpha _{p}=\frac{\tilde{\alpha}_{p}}{f(t_{p})}$ and $g/f \in C(\Omega)$ can be approximated by $\sum_{p=1}^{N}\tilde{\alpha} _{p}k(x,t_{p})$ to an arbitrary accuracy when $N$ is large enough. Thus $\tilde{k}$ is also universal. By Lemma \ref{lem:1} below, all universal kernels defined on the compact metric space $\Omega$ metrize the same topology on $P(\Omega)$ as , i.e., the weak topology on $P(\Omega)$, which is the weakest topology such that the map $\mu \mapsto \int  _{\Omega} g \, d\mu$ is continuous for all $f \in C(\Omega)$. Then for the universal kernels $k',\tilde{k},$ $(P(\Omega), \gamma _{k'})$ and $(P(\Omega), \gamma _{\tilde{k}})$ share the same topology on $P(\Omega)$. It's also easy to observe that the kernel embedding $\tilde{k}_{\mathbf{P}}$ is an isometry between $(P(\Omega),\gamma _{\tilde{k}})$ and $\big(\tilde{k}_{\mathbf{ P}}(P(\Omega)), \|\cdot\|_{\mathcal{H}_{\tilde{k}}} \big)$, which follows that $\tilde{k}_{\mathbf{p}}$ is a continuous mapping from $(P(\Omega), \gamma _{k'})$ to $(\mathcal{H}_{\tilde{k}}, \|\cdot\|_{\mathcal{H}_{\tilde{k}}})$. Then it can be concluded by ({\ref{attn}}) that $\operatorname{attn}$ is also a continuous mapping from $(P(\Omega), \gamma _{k'})$ to $(\mathcal{H}_{k}, \|\cdot\|_{\mathcal{H}_{k}})$.
\end{proof}

\subsection{Proof of Theorem 2}
\begin{proof}
    First we have the following error decomposition
\begin{align}\label{error}
    	\left\| \Phi _{k,f}-H_{n_{1},n_{2}} \right\|_{\rho} \leq \left\| \Phi _{k,f}- H_{n_{1}} \right\|_{\rho} + \left\| H_{n_{1}} - H_{n_{1},n_{2}} \right\|_{\rho}
\end{align}
where $H_{n_{1}}(\mu)$ has the form of $\sum_{p=1}^{n_{1}}c_{p}\phi(\langle  a_{p}, g \rangle_{\mathcal{H}_{k}}+b_{p})+b_{0}$ with $c_p, b_p, b_0 \in \mathbb{R}$, $a_p \in \mathcal{H}_{k}$ for all $1 \leq p \leq n_1$ and a certain sigmoidal function $\phi$.  
In the following we present upper bounds for the two terms on RHS of (\ref{error}) respectively. 
\vspace{1em}

\textbf{Step 1.}
An upper bound for $\left\| \Phi _{k,f}-H_{n_{1}} \right\|^{2}_{\rho}$ is derived with the help of Lemma 1. We specify the sigmoidal function and the probability measure in our case as follows.
Let $\phi(x)$ be the sigmoidal function $\sigma\left( x+\frac{1}{2} \right) - \sigma\left( x-\frac{1}{2} \right)$. Let $\beta \ge 1$. Then for ReLU neural networks, 
\begin{align}
    \eta _{\beta}=\inf _{0<\epsilon \leq \frac{1}{2}}\left\{  2 \epsilon + \sup _{|z|\geq \epsilon} \left| \sigma\left( \beta z +\frac{1}{2} \right)-\sigma\left( \beta z -\frac{1}{2} \right) - 1_{\{ z >0 \}} \right|  \right\}
\end{align}
with $\left| \sigma\left( \beta z+\frac{1}{2} \right)-\sigma\left( \beta z -\frac{1}{2} \right) -1_{\{ z >0 \}}\right| \leq \frac{1}{2} -\beta\epsilon$ for $\epsilon \le \left| z \right| \le \frac{1}{2\beta}$ and is $0$ for $|z| \ge \frac{1}{2\beta}$. Take $\epsilon=\frac{1}{2\beta}$ then we have $\eta _{\beta}\leq \frac{1}{\beta}$. 

By Theorem 1,  $\operatorname{attn}:(P(\Omega),\gamma _{k}) \to (\mathcal{H}_{k}, \|\cdot\|_{k})$ is continuous. Then Borel probability measure $\rho _{\mathcal{U}}$ on $P(\Omega)$ defines another Borel probability measure $\tilde{\mu}$ on $\mathcal{H}_{k}$ by $\tilde{\mu}(\mathcal{B}):=\rho _{\mathcal{U}}(\operatorname{attn}^{-1}(\mathcal{B}))$ where $\mathcal{B}$ is a Borel set in $\mathcal{\mathcal{H}}_{k}$. Note that $\mathcal{G}_{k,f}$, the image of the attention operator, is contained in the closed ball $B_{r}$. Then we can denote the restriction of $\tilde{\mu}$ on $B_{r}$ by $\nu$.

Recall that $\Phi _{k,f}$ has the form of $\Phi _{k,f}(\mathcal{P})= \Phi\left( \int  _{\Omega} k(x,\cdot)f(x) \, d\mathcal{P} \right)$ with $\Phi \in \Gamma _{r,C}(\mathcal{H}_{k})$. Then by taking $\beta=n_{1}^{\frac{1}{2}}$ and the probability measure $\nu$,  and applying Lemma 1, we get a functional defined on $B_{r}$ by $\widetilde{ H}_{n_{1}}(g)= \sum_{p=1}^{n_{1}}c_{p}\phi\left( \left\langle  a_{p}, g \right\rangle_{\mathcal{H}_{k}}+b_{p} \right)$ with $\left\| a_{p} \right\|_{\mathcal{H}_{k}} \leq \frac{n_{1}^{1/2}}{r}, \left| b_{p} \right|\leq n_{1}^{1/2}$ for all $p$ and $\|c\|_{1}\leq 2rC$ such that 
$$
\int _{\mathcal{H}_k} (\Phi(g)-\widetilde{H}_{n_{1}}(g))^{2} \tilde{\mu} (dg)= \int_{B_{r}} (\Phi(g)-\widetilde{H}_{n_{1}}(g))^{2} \, \nu (dg)  \leq \frac{(4rC)^{2}}{n_{1}},$$
which follows that 
\begin{align}
\int _{P(\Omega)} (\Phi _{k,f}(\mathcal{P})-H_{n_{1}}(\mathcal{P}))^{2} \, d\mu \leq \frac{(4rC)^{2}}{n_{1}}
\end{align}
with $H_{n_{1}}(\mathcal{P})= \widetilde{H}_{n_{1}}\left( \int  _{\Omega}k(x, \cdot) f(x)\, d\mathcal{P} \right)$.

\vspace{1em}

\textbf{Step 2.} We discretize the weight parameters $\{a_p\}$ in the network by kernel trick.

Observe that 
\begin{align}
    H_{n_{1}}(\mathcal{P})=\sum_{p=1}^{n_{1}}c_{p}\phi\left( \left\langle  a_{p}, \int  _{\Omega}k(x, \cdot)f(x) \, d\mathcal{P} \right\rangle_{\mathcal{H}_{k}}+b_{p} \right)
\end{align}
and that 
\begin{align}
    \left\langle  a_{p}, \int _{\Omega} k(x, \cdot) f(x)\, d\mathcal{P}  \right\rangle_{\mathcal{H}_{k}}=\int  _{\Omega} \langle  a_{p},k(x, \cdot) \rangle_{\mathcal{H}_{k}} \,f(x) d\mathcal{P} = \int _{\Omega} a_{p}(x) f(x)\, d\mathcal{P}
\end{align}
with $\|a_{p}\|_{\mathcal{H}_k}\leq \frac{n_{1}^{1/2}}{r}$ for all $1 \leq p \leq k$. 

\vspace{1em}
For $(a'_p)_{p=1}^{n_1} \subset \mathcal{H}_k$, we have from the Lipschitz continuity of $\phi$ that
\begin{align*}
& \left| \sum_{p=1}^{n_{1}}c_{p}\phi\left( \int _{\Omega} a_{p}(x) f(x)\, d\mathcal{P}+b_{p} \right) - \sum_{p=1}^{n_{1}}c_{p}\phi\left( \int _{\Omega} a_{p}'(x) f(x)\, d\mathcal{P}+b_{p} \right)\right|  \\
& \leq \|c\|_{1} \max _{p} \left| \phi\left(\int _{\Omega} a_{p}(x) f(x)\, d\mathcal{P}+b_{p} \right) - \phi\left(\int _{\Omega} a_{p}'(x)f(x) \, d\mathcal{P}+b_{p} \right)\right|  \\ 
 & \leq 2\|c\|_{1} \max _{p} \left| \int  _{\Omega} [a_{p}(x) - a_{p}' (x)]f(x)\, d\mathcal{P}  \right|  \\ 
& \leq 2 \|c\|_{1} C_{f}\max _{p} \sup_{x \in \Omega} \left| a_{p}(x)-a_{p}'(x) \right|. 
\end{align*}

To apply kernel trick and get discrete approximations to network parameters $(a_p)_{p=1}^{n_1}$, we take $n \in \mathbb{N}$ and 
$$
\mathbf{T}=\left\{  (t ^{(1)}, \dots , t ^{(d)}): t ^{(j)} \in \left\{  -1 + \frac{2l}{n}  \right\}_{l=0}^{n} \text{ for }1 \leq j \leq d  \right\}
$$
to be the uniform mesh on $\Omega=[-1,1]^{d}$.

We choose 
\begin{equation} \label{eq:ap}
    a_p^{n_2}(x):=\sum_{q=1} ^{n_2} a_p(t_{q})u_{q}(x)
\end{equation}
 with $n_2 = |\mathbf{T}|$ and $$(u_{q})_{q=1}^{n_{2}}= [(k(x_{i},x_{j}))_{x_{i},x_{j} \in \mathbf{T}}]^{-1}(k(\cdot, x_{i}))_{x_{i}\in \mathbf{T}}$$ to be the so-called nodal functions satisfying 
\begin{align}
   \left| a_{p}(x) - a_{p}^{n_{2}}(x) \right| \leq \left\| a_{p} \right\|_{\mathcal{H}_{k}} \exp\left( -\frac{c_{k}n_{2}^{1/d}}{2\sqrt{ d }} \right), \forall x \in \Omega.
\end{align}
This can be found in \citep{zhou2024approximation} with a constant $c_k$ depending on the kernel $k$.

\vspace{1em}

\textbf{Step 3.}
Let $H_{n_{1},n_{2}}(\mathcal{P})=\sum_{p=1}^{n_{1}}c_{p}\phi\left( \left\langle  a_{p}^{n_{2}}, \int  _{\Omega}k(x, \cdot) f(x)\, d\mathcal{P} \right\rangle_{\mathcal{H}_{k}}+b_{p} \right)$. Then we can obtain that 
$$
 \left\| H_{n_{1}} - H_{n_{1},n_{2}} \right\|^{2}_{\rho} = \int  \left(  H_{n_{1}}(\mathcal{P})-H_{n_{1},n_{2}} (\mathcal{P}) \right)^{2} \, d\rho \leq 16C^{2}C_{f}^{2}n_{1} \exp \left( -\frac{c_{k}n_{2}^{1/d}}{\sqrt{ d }} \right).   $$
Here we have used the bound $\|c\|_1 \le 2rC$ and $\|a_p\|_{\mathcal{H}_k} \leq \frac{\sqrt{n_1}}{r}$.

Combining this error bound with the estimate in \textbf{Step 1} and $r=C_f$ ($\|k\|_{\infty}=1$ for Gaussian kernels) results in
$$
\left\| \Phi _{k,f}-H_{n_{1},n_{2}} \right\|_{\rho} \leq \frac{4C_{f}C}{n_{1}^{1/2}} + 4CC_{f}n_{1}^{1/2}\exp\left( -\frac{c_{k}n_{2}^{1/d}}{2\sqrt{ d }} \right).  $$
Let $n_{2}=\lceil C_{k,d}(\log n_{1})^{d} \rceil$ with $C_{k,d}=\left( \frac{2\sqrt{ d }}{c_{k}} \right)^{d}$. Then we have the RHS of the above bounded as
$$
 \left\|  \Phi _{k,f} - h_{n} \right\|_{\rho} \leq \frac{(8C_{f})C}{n^{1/2}} \textrm{ where } h_{n} := H_{n_{1},n_{2}} \textrm{ with }n_{1} = n \textrm{ and } n_{2}=\lceil C_{k,d}(\log n)^{d} \rceil. $$

Insert the linear expressions of $\{a_{p}^{n_{2}}\}$ in (\ref{eq:ap}) into $H_{n_{1},n_{2}}$ and it can be obtained that 
$$
H_{n_{1},n_{2}}(\mathcal{P}) = \sum^{n_{1}}_{p=1}c_{p}\phi \left( \sum^{n_{2}}_{q=1}A_{p,q}\int  _{\Omega}k(x, t_{q}) f(x)\, d \mathcal{P}  +b_{p}\right)$$
where $\{ A_{p,q} \}$ is determined by the values of $a_{p}$ on $\{ t_{q} \}$ and $K_{T}$ and furthermore, $a_{p,q}$ can be bounded by $\sqrt{ n_{2} }\left\| K_{T}^{-1} \right\|_{2}n_{1}^{1/2}$ with $K_{\mathbf{T}}=(k(x_i,x_j))_{x_i,x_j \in \mathbf{T}}$.
In conclusion, with $O(n(\log n)^{d})$ parameters, we can achieve the error bound $\left\| \Phi _{k,f} - h_{n} \right\|^{2}_{\rho} = O(n^{-1})$ with $\|c\|_{1} \leq 2C_{f}C$, $\left| A_{p,q} \right| \leq \sqrt{ 2C_{k,d}(\log n)^{d} }\left\| K_{T} ^{-1}\right\|_{2}n^{1/2}$ and $\left| b_{p} \right| \leq n^{1/2}$ for all $p,q$.

\smallskip

Let $h_{n}(\mathcal{P}):=c^{T}\sigma \left(  A[\operatorname{attn}(\mathcal{P})]_{\mathbf{T}_{s(n)}} +b\right)$ with $c \in \mathbb{R}^{2n}, A \in \mathbb{R}^{2n \times s(n)}, b \in \mathbb{R}^{2n}$ and $\mathbf{T}_{s(n)}$ the set of $s(n)$ fixed points. Recall that $\phi(x) = \sigma(x + \frac{1}{2}) - \sigma(x - \frac{1}{2})$ . Now we define $$
\mathcal{H}'_{R,n}=\left\{ h_{n}: \|c\|_{1} \leq 4rC, \left| A_{p,q} \right| \leq \sqrt{ 2C_{k,d}(\log n)^{d} }\left\| K_{T} ^{-1}\right\|_{2}n^{1/2} \text{ and } \left| b_{p} \right| \leq 2n^{1/2} \textrm{ for all } p,q\right\}.$$
Since $\sigma$ is homogeneous, 
$$
 \mathcal{H}'_{R,n} = \left\{ h_{n}: \|c\|_{1} \leq 8rCn^{1/2}, \left| A_{p,q} \right| \leq \sqrt{ \frac{1}{2}C_{k,d}(\log n)^{d} }\left\| K_{T} ^{-1}\right\|_{2} \text{ and } \left| b_{p} \right| \leq 1 \textrm{ for all } p,q\right\} $$

By Example 1 in \citep{zhou2003capacity}, for the case of the Gaussian kernel, an upper bound can be derived for $\left\| K_{T}^{-1} \right\|_{2}$ that 
$$
 \left\| K_{T}^{-1} \right\|_{2} \leq C_{1} (\log n)^{-d} n^{C_{2}\log n}   $$
where $C_{1}=(\alpha \sqrt{ \pi })^{-d}C_{k,d}^{-1}$ and $C_{2}=d\pi^{2}\alpha^{2}C_{k,d}^{2/d}$.
Let $R = \max \{ \sqrt{ \frac{1}{2}C_{k,d} }C_{1},C_{2}, 8rC,1\}$. Then take the hypothesis space to be 
$$
  \mathcal{H}_{R,n}=\{ h_{n}: \|c\|_{1} \leq Rn^{1/2}, \left| A_{p,q} \right| \leq R(\log n)^{-d/2}n^{R\log n} \text{ and } \left| b_{p} \right| \leq R \textrm{ for all } p,q\}.  $$
  
  We see the conclusion of Theorem 2.
\end{proof}

\subsection{Proof of Theorem 3}
\begin{proof}
    Since $C(P(\Omega))$ is a metric space, it suffices to prove that the hypothesis space is a sequentially compact subset. Let $\{ h^{(j)} \}$ be a countable collection of functions in the hypothesis space $\mathcal{H}_{R,n}$. By Lemma 5 in the appendix, it suffices to show that the functions $\{ h^{(j)} \}$ are equi-bounded and equi-continuous.
\smallskip

\noindent\textbf{Equi-continuity:} For $\mathcal{P}, \mathcal{Q} \in (P(\Omega), \gamma _{k})$, we have 
\begin{align*}
		 \left| h^{(j)}(\mathcal{P})-h^{(j)}(\mathcal{Q}) \right|  & \leq \|c^{(j)}\|_{1} \max _{p} \left| \sum^{n_{2}}_{q=1}A^{(j)}_{p,q}\int  _{\Omega}k(x, t_{q})f(x) \, d \mathcal{P} - \sum^{n_{2}}_{q=1}A^{(j)}_{p,q}\int  _{\Omega}k(x, t_{q})f(x) \, d \mathcal{Q} \right| \\
		 & \leq \|c^{(j)}\|_{1} \max _{p} \left\{ \left(\sum^{n_{2}}_{q=1}\left| A^{(j)}_{p,q} \right|\right) \right\} \max _{q}\left| \int  _{\Omega} k(x,t_{q}) f(x) \, d (\mathcal{P}-\mathcal{Q})  \right| \\
 & \leq \left\|  c^{(j)} \right\|_{1} \left\|  k \right\|_{\infty}  \max _{p}  \left\{ \left(\sum^{n_{2}}_{q=1}\left| A^{(j)}_{p,q} \right|\right)\right\} \left\| \int  _{\Omega} k(x, \cdot) f(x) \, d (\mathcal{P}-\mathcal{Q})  \right\|_{\mathcal{H}_{k}} 
\end{align*}
By Theorem 1, $\operatorname{attn}$ is a continuous mapping from $(P(\Omega), \gamma _{k})$ to $(\mathcal{H}_{k}, \|\cdot\|_{\mathcal{H}_{k}})$, which implies that $\{h^{(j)}\}$ are equi-continuous.

\smallskip
\noindent\textbf{Equi-boundedness:} For a collection of functions $\{ h^{(j)} \} \subset \mathcal{H}_{R,n}$ with $h^{(j)}: (P(\Omega), \gamma _{k}) \to (\mathbb{R}, |\cdot|)$, it's sufficient to show that $\{ h^{(j)} \}$ is uniformly bounded. For any $n \in \mathbb{N}$, it's easy to obtain that 
\begin{align*}
\sup _{\mathcal{P} \in (P(\Omega), \gamma _{k})} \left| h^{(j)}(\mathcal{P})  \right| &\leq \|c^{(j)}\|_{1} \max_{p} \left[ \left(\sum^{n_{2}}_{q=1}\left| A^{(j)}_{p,q} \right|\right) C_{f} \left\| k \right\|_{\infty}^{2}  + \left| b_{p}^{(j)} \right| \right]\\
&\leq R^{2}n^{\frac{1}{2}} \left(  C_{f}s(n)(\log n)^{-d/2}n^{R\log n}+1 \right)
\end{align*}
Therefore, $\mathcal{H}_{R,n}$ is a compact subset of $C(P(\Omega))$. This completes the proof of Theorem 3.
\end{proof}

Our generalization analysis needs an estimate of the covering numbers of the hypothesis space $\mathcal{H}_{R,n}$, which is given by the following lemma.

\begin{lemma}
    For $n \geq 3$, the covering number $\mathcal{N}(\mathcal{H}_{R,n}, \epsilon, \|\cdot \|_{\infty})$ induced by the Transformer encoders can be bounded as
    \begin{align}
\log \mathcal{N}(\mathcal{H}_{R,n}, {\epsilon},  \| \cdot\|_{\infty}) \leq R_{1}n(\log n)^{d} \log\left( \frac{\tilde{ R}}{{ \epsilon}} \right) + R_{2}n(\log n)^{d+2}
\end{align}  
where $R_{1}:=6C_{k,d}$, $R_{2}:=2(8R+3d)$ and $\tilde{R}:=6R^2(1+2C_{k,d})\max \{ 1, 2C_{f}(C_{k,d}+1)\}$.
\end{lemma}
\begin{proof}
    For $h \in \mathcal{H}_{R,n}$, denote $\sigma _{h}(\mathcal{P}):=\sigma(A[\int_{\Omega} k(x, \cdot)f(x) \, d\mathcal{P}]_{\mathbf{T}_{s(n)}}+b)$ with the parameters $A$ and $b$ in $h$. Then we have 
\begin{align*}
&\left\|\sigma _{h}\right\|_{\infty} \leq \left\| A\right\|_{\infty} \sup _{\mathcal{P}\in (P(\Omega), \gamma _{k})}\left\| \int_{\Omega} k(x, \cdot) f(x)\, d\mathcal{P}\right\|_{\infty} + \left\| b \right\|_{\infty}  \\ 
&\leq s(n) R(\log n)^{-d/2}n^{R\log n}\left\| k \right\|_{\infty}^{2}C_{f} +R  \\
&\leq 2C_{k,d}\left\| k \right\|_{\infty}^{2}C_{f}R(\log n)^{d/2}n^{R\log n}+R \\
& \leq (2C_{f}C_{k,d}+1)R(\log n)^{d/2}n^{R\log n}.
\end{align*}

Let $\hat{h}$ be another function in $\mathcal{H}_{R,n}$ with parameters $\hat{c},  \hat{A}_{p,q}$ and $\hat{b}$ such that $\|c-\hat{c}\|_{\infty} \leq \epsilon, \left| A_{p,q}-\hat{A}_{p,q} \right| \leq \epsilon \textrm{ and }\left\| b - \hat{b} \right\|_{\infty} \leq \epsilon$. Then we have that 
\begin{align*}
\left\|  \sigma_{h} - \sigma_{\hat{h}}  \right\|_{\infty} &\leq \left\| A - \hat{A} \right\|_{\infty} \sup _{\mathcal{P}\in (P(\Omega), \gamma _{k})}\left\| \int_{\Omega} k(x, \cdot)f(x) \, d\mathcal{P}\right\|_{\infty} + \left\| b-\hat{b} \right\|_{\infty} \\
&\leq (s(n)C_{f} +1)\epsilon
\end{align*}
which results in a bound on the final output, 

\begin{align*}
\left\|  h-\hat{h} \right\|_{\infty}&=\left\| c^{T}\sigma_{h}-\hat{c}^{T}\sigma_{\hat{h}} \right\|_{\infty} \leq  \left\| (c-\hat{c})^{T}\sigma_{h} \right\|_{\infty} + \left\| \hat{c}^{T}(\sigma_{h}-\sigma_{\hat{h}})  \right\|_{\infty}   \\
&\leq 2\epsilon n (2C_{f}C_{k,d} +1)R(\log n)^{d/2}n^{R\log n} + Rn^{1/2}(s(n)C_{f}+1)\epsilon  \\
&\leq 2C_{3}R\ \epsilon\  n(\log n)^{d/2}n^{R\log n} + 2C_{3} R \ \epsilon \ n^{1/2} s(n) \\
& \leq C_{4}(\log n)^{d} n^{2R\log n} \epsilon
\end{align*}  
for ${ n \geq 3}$ where $C_{3}:= \max \{ 2, 4C_{f}(C_{k,d}+1)\}$ and $C_{4}:=2C_{3}R(1+2C_{k,d})$.
Let $\tilde{\epsilon}=C_4(\log n)^d n^{2R\log n}\epsilon$. Then the $\tilde{\epsilon}$-covering number of $\mathcal{H}_{R,n}$ can be bounded

\begin{align*}
&\mathcal{N}(\mathcal{H}_{R,n}, \tilde{\epsilon}, \|\cdot\|_{\infty}) \leq \left\lceil \frac{2Rn^{1/2}}{\epsilon} \right\rceil^{2n} \left\lceil{\frac{2R(\log n)^{-d/2}n^{R\log n}}{\epsilon}}\right\rceil^{2ns(n)} \left\lceil{\frac{2R}{\epsilon}}\right\rceil^{2n} \\
&\leq \left( \frac{\tilde{R}}{\tilde{\epsilon}} \right)^{4n+2ns(n)} (\log n)^{6dns(n)} n^{16Rn(\log n) s(n)} \\
\end{align*}
with $\tilde{R}:=3RC_{4}$. Then by taking logarithms on both sides, we obtain
\begin{align*}
\log \mathcal{N}(\mathcal{H}_{R,n}, \tilde{\epsilon},  \| \cdot\|_{\infty}) &\leq [4n+2ns(n)]\log \left( \frac{\tilde{R}}{\tilde{\epsilon}} \right)+6dns(n)\log(\log n)+16Rn(\log n)s(n)\log n \\
&\leq R_{1}n(\log n)^{d} \log\left( \frac{\tilde{ R}}{\tilde{ \epsilon}} \right) + R_{2}n(\log n)^{d+2}
\end{align*}  
where $R_{1}:=12C_{k,d}$ and $R_{2}:=4(8R+3d)$.
\end{proof}

\subsection{Proof of Theorem 4}
\begin{proof}
    We apply the following concentration inequalities given in \citep{yu2023deep} for $\mathcal{I}_1(D,\mathcal{H}_{R,n})$ and $\mathcal{I}_2(D,\mathcal{H}_{R,n})$:

\begin{align}
&Prob\left\{ \mathcal{I}_1(D,\mathcal{H}_{R,n}))> \frac{1}{2} \left( \mathcal{E}(\pi _{M}\varphi _{\hat{D},R,n})- \mathcal{E}(\varphi_{\rho})  \right) +\epsilon \right\} \notag\\
& \,\,\, \leq \mathcal{N}\left(  \mathcal{H}_{R,n}, \frac{\epsilon}{16M}, \|\cdot\|_{\infty} \right) \exp \left\{  -\frac{3m_{1}\epsilon}{2048M^{2}}  \right\}  
\end{align}
and
\begin{equation}
    Prob \{ \mathcal{I}_2(D,\mathcal{H}_{R,n}))> \epsilon \} \leq \exp \left\{   -  \frac{m_{1}\epsilon^{2}}{2(3M+\left\| h \right\|_{\infty} )^{2}\left( R(\mathcal{H})+\frac{2}{3}\epsilon \right)}  \right\}. 
\end{equation}

For $\mathcal{I}_3(D, \mathcal{H}_{R,n})$, since $\left\| \pi _{M}\varphi_{\hat{D},R,n} \right\|_{\infty} \leq M$ and $|y_{i}| \leq M$, there holds

\begin{align*}
\left| I_{3}(\hat{ D},\mathcal{H}_{R,n}) \right|&=\left| \frac{1}{m_{1}} \sum_{i=1}^{m_{1}} \left(  \pi _{M}\varphi_{\hat{D},R,n}(\hat{\mu}_{i}^{m_{2}})-y_{i} \right)^{2}- \left(  \pi _{M}\varphi_{\hat{D},R,n}(\mu _{i})-y_{i} \right)^{2} \right| \\
&\leq \frac{1}{m_{1}}\sum^{m_{1}}_{i=1} 4M \left| \varphi_{\hat{D},R,n}(\hat{\mu}_{i}^{m_{2}}) - \varphi_{\hat{D},R,n}(\mu _{i}) \right|
\end{align*} 
where $\varphi_{\hat{D},R,n}(\mu )=c_{\varphi}^{T}\left(\sigma\left(A_{\varphi}\left[\int_{\Omega} k(x, \cdot) f(x) \, d \mu\right]_{\mathbf{ T}_{s(n)}}+b_{\varphi}\right)\right)$ and $\mathbf{T}_{s(n)}$ denotes a set of $s(n)$ distinct points in $\Omega$. 

Let $g_{t}(x):= k(x, t) f(x)$ for $t \in \mathbf{T}_{s(n)}$. It can be derived that $\left| g_{t}(x) \right| \leq C_f$ for any $x \in \Omega$. For each $t \in \mathbf{T}_{s(n)}$, we conclude that 
$$
Prob \{ \left| \mathbb{E}(g_{t}(X_{i})) - S_{m_{2}}(g_{t}(X_{i,j})) \right| > \epsilon \} \leq 2 \exp \left\{  - \frac{m_{2} \epsilon^{2}}{8C_f^{2}} \right\},
$$
which follows that 
$$
 Prob \left\{  \sup_{t \in \mathbf{T}_{s(n)}} \left| \mathbb{E}(g_{t}(X_{i})) - S_{m_{2}}(g_{t}(X_{i,j})) \right| > \epsilon   \right\} \leq 2s(n)\exp \left\{  - \frac{m_{2} \epsilon^{2}}{8C_f^{2}}  \right\}. $$
Since 
\begin{align*}
\sup _{\varphi \in \mathcal{H}_{R,n}} \left| \varphi(\mu _{i}) - \varphi(\hat{\mu}_{i}^{m_{2}}) \right| &\leq \sup _{\varphi \in \mathcal{H}_{R,n}} \|c_{\varphi}\|_{1} \left\| A \right\|_{\infty} \sup _{t \in T_{n}}  \left| \mathbb{E}(g_{t}(X_{i})) - S_{m_{2}}(g_{t}(X_{i,j})) \right| \\
& \leq Rn^{1/2} s(n)R(\log n)^{-d/2}n^{R\log n} \epsilon \\
& \leq C_{5}n^{2R\log n} (\log n)^{d/2} \epsilon
\end{align*}
where $C_{5}:=2C_{k,d}R^{2}$, it can be concluded that 
$$
Prob \left\{  \sup_{\varphi \in \mathcal{H}_{R,n}} \left| \varphi(\mu _{i}) - \varphi(\hat{\mu}_{i}^{m_{2}}) \right| > C_{5}n^{2R\log n}(\log n)^{d/2} \epsilon \right\}  \leq 2s(n) \exp \left\{   - \frac{m_{2} \epsilon^{2}}{8C_f^{2}}  \right\} $$
which results in the probability concentration of $\mathcal{I}_3(\hat{D},\mathcal{H}_{R,n})$ as
$$
 Prob \left\{  \left| I_{3} (\hat{D},\mathcal{H}_{R,n})>\epsilon\right| \right\} \leq 2m_{1}s(n) \exp \left\{   - \frac{m_{2} \epsilon^{2}}{128M^{2}C_f^{2}C_{5}^{2}n^{4R\log n}(\log n)^{d}}  \right\} $$

Similarly, it can be obtained that 
$$
 Prob \left\{  \left| I_{4} (\hat{D},\mathcal{H}_{R,n})>\epsilon\right| \right\} \leq 2m_{1}s(n) \exp \left\{   - \frac{m_{2} \epsilon^{2}}{32(M^{2}+\left\| h \right\|_{\infty}^{2}) C_f^{2}C_{5}^{2}n^{4R\log n}(\log n)^{d}}  \right\}  $$
Finally, combine all the probability concentration inequalities. We can derive 
\begin{align*}
& Prob\left\{\left\|\pi_M \varphi_{\hat{D}, R, n}-\varphi_\rho\right\|_\rho^2>2\left\|h-\varphi_\rho\right\|_\rho^2+8 \epsilon\right\} \\
\leq & \mathcal{N}\left(\mathcal{H}_{R, n}, \frac{\epsilon}{16 M},\|\cdot\|_{\infty}\right) \exp \left\{-\frac{3 m_{1} \epsilon}{2048 M^2}\right\} \\
+ & \exp \left\{-\frac{m_{1} \epsilon^2}{2\left(3 M+\|h\|_{\infty}\right)^2\left(\left\|h-\varphi_\rho\right\|_\rho^2+\frac{2}{3} \epsilon\right)}\right\}\\
+ &4m_{1}s(n) \exp \left\{  - \frac{m_{2} \epsilon^{2}}{128\max \{ \left\| h \right\|_{\infty}^{2},M^{2}  \}C_f^{2}C_{5}^{2}n^{4R\log n}(\log n)^{d}}  \right\}.
\end{align*}
This proves Theorem 4.
\end{proof}

\subsection{Proof of Theorem 5}
\begin{proof}
    By the construction of the approximation with sigmoidal functions in Theorem 1, we have $\left\| h \right\|_{\infty} \leq 2C_fC$.

By inserting the estimations for the approximation error and the covering number into Theorem 4, it can be obtained that 

\begin{align*}
&\operatorname{Prob} \left\{ \left\| \pi _{M} \varphi _{\hat{D},R,n} - \varphi _{\rho} \right\|_{\rho}^{2} > 2C_{*}^{2}n^{-1} +8\epsilon \right \} \\
& \leq \exp \left\{  R_{1}n(\log n)^{d} \log\left( \frac{{16M \tilde{R}}}{\epsilon} \right) +R_{2}n(\log n)^{d+2} - \frac{{3m_{1} \epsilon}}{2048M^{2}} \right\} \\
&+\exp \left\{  - \frac{m_{1}\epsilon^{2}}{2(3M+2C_fC)^{2}\left( C_{*}^{2}n^{-1}+\frac{2}{3}\epsilon \right)}  \right\} \\
&+ \exp \left\{  \log(8C_{k,d}m_{1}) + d \log(\log n) - \frac{m_{2}\epsilon^{2}}{128(2C_fC+M)^{2}C_f^{2}C_{5}^{2}n^{4R\log n}(\log n)^{d}}  \right\}.
\end{align*}

If $\epsilon \geq 2C_{*}^{2}n^{-1}(\log n)^{d+2}$, then we have 
\begin{align*}
&\operatorname{Prob} \left\{ \left\| \pi _{M} \varphi _{\hat{D},R,n}- \varphi _{\rho} \right\|^{2}_{\rho} > 9\epsilon \right\} \\
&\leq \exp \left\{  R_{1}n(\log n)^{d} \log \left( \frac{8M\tilde{R}n}{C_{*}^{2}} \right) + R_{2}n(\log n)^{d+2} - \frac{3m_{1}\epsilon}{2048M^{2}}   \right\} \\
&+ \exp \left\{  - \frac{3m_{1}\epsilon}{8(3M+2C_fC)^{2}}   \right\} \\
&+ \exp \left\{  \log(8C_{k,d}m_{1})+ d\log(\log n) - \frac{m_{2} \epsilon^{2}}{128(2C_fC+M)^{2}C_f^{2}C_{5}^{2}n^{4R\log n}(\log n)^{d}}   \right\} \\
& \leq \exp \left\{   \mathcal{A}_{1}n(\log n)^{d+2} - \frac{3m_{1}\epsilon}{2048M^{2}}  \right\} + \exp \left\{  - \frac{3m_{1}\epsilon}{8(3M+2C_fC)^{2}}  \right\} \\
&+ \exp \left\{  \log (8C_{k,d}(\log n)^{d}m_{1}) - \frac{m_{2}\epsilon^{2}}{\mathcal{A}_{2}n^{4R\log n}(\log n)^{d}}  \right\}
\end{align*}
where $\mathcal{A}_{1}:=R_{1}\left( \log \left( \frac{8M\tilde{R}}{C_{*}^{2}} \right) +1\right)+R_{2}$ and $\mathcal{A}_{2}:=128(2C_fC+M)^{2}C_f^{2}C_{5}^{2}$.
If we choose the neural network parameter $n = [\mathcal{A}_{3}m_{1}^{1/2}]$ with $\mathcal{A}_{3}:=(\frac{3C_{*}^{2}}{2048\mathcal{A}_{1}M^{2}})^{1/2}$, then we have that when $1 \leq \log(8C_{k,d}m_{1}) \leq \frac{3\mathcal{A}_{4}^{-1}\mathcal{A}_{3}^{-1}C_{*}^{2}}{4096M^{2}}m_{1}^{1/2}$,
 \begin{align*}
\operatorname{Prob} \left\{  \left\| \pi _{M} \varphi _{\hat{D},R,n} -\varphi _{\rho}\right\|_{\rho}^{2} > 9 \epsilon  \right\} &\leq \exp \left\{   \frac{{3m_{1}\epsilon}}{4096M^{2}} - \frac{3m_{1}\epsilon}{2048M^{2}} \right\} + \exp \left\{   - \frac{{3m_{1} \epsilon}}{8(3M+2C_fC)^{2}}  \right\} \\
&+ \exp \left\{  \log(8C_{k,d}m_{1}) + d\log \log(\mathcal{A}_{3}m_{1}^{1/2}) - \frac{m_{2}\epsilon^{2}}{\mathcal{A}_{2}n^{4R\log n}(\log n)^{d}}   \right\}  \\
& \leq \exp \left\{  - \frac{{3m_{1} \epsilon}}{4096M^{2}}  \right\} + \exp \left\{   - \frac{{3m_{1} \epsilon}}{8(3M+2C_fC)^{2}}  \right\}  \\
&+ \exp \{  \mathcal{A}_{4} \log (8C_{k,d}m_{1}) -  \frac{m_{2}\epsilon^{2}}{\mathcal{A}_{2}n^{4R\log n}(\log n)^{d}} \}
\end{align*}
where $\mathcal{A}_{4}:= 2d (1+ \log(\log \mathcal{A}_{3}))$. Take $m_{2} \geq \left\lceil{ \mathcal{A}_{5} \left( \mathcal{A}_{3}m_{1}^{ \frac{1}{2} } \right)^{8R\log\left( \mathcal{A}_{3}m_{1}^{\frac{1}{2}} \right)}}\right\rceil$ with $\mathcal{A}_{5}:=\frac{3m_{1}}{4096C_{*}^{2}M^{2}}\mathcal{A}_{2}$. Then
\begin{align*}
\operatorname{Prob} \left\{  \left\| \pi _{M} \varphi _{\hat{D},R,n} -\varphi _{\rho}\right\|_{\rho}^{2} > 9 \epsilon  \right\} &\leq \exp \left\{  - \frac{{3m_{1} \epsilon}}{4096M^{2}}  \right\} + \exp \left\{   - \frac{{3m_{1} \epsilon}}{8(3M+2C_fC)^{2}}  \right\}  \\
&+ \exp \left\{  \frac{{3m_{1} \epsilon}}{4096M^{2}} - \frac{{3m_{1} \epsilon}}{2048M^{2}} \right\}
 \\
&\leq 3 \exp \left\{   - \frac{{3m_{1} \epsilon}}{256(4M+2C_fC)^{2}}  \right\}.
\end{align*}
Let $\tau = 9 \epsilon$ and it can be obtained that 
$$\operatorname{Prob} \left\{ \left\| \pi _{M} \varphi _{\hat{D},R,n} - \varphi _{\rho}\right\|_{\rho}^{2} > \tau \right\} \leq 3\exp \left\{  - \frac{m_{1}^{\frac{1}{2}}\tau}{768(4M+2C_fC)^{2}}  \right\}$$ for $\tau \geq 18 C_{*}^{2}n^{-1}(\log n)^{d+2}$.
Then by the formula for the expectation of the non-negative random variable $\xi$, $\mathbb{E} \xi =\int  _{0}^{\infty} P(\xi >\tau) \, d \tau$, we get

\begin{align*}
&\mathbb{E}\{ \mathcal{E}(\pi _{M}\varphi_{\hat{D},R,n})-\mathcal{E}(\varphi _{\rho}) \}=\int_{0} ^{\infty} \operatorname{Prob} \{ \left\| \pi _{M} \varphi _{\hat{D},R,n}-\varphi _{\rho} \right\|_{\rho}^{2} > \tau \}   \, d \tau   \\
& = \left( \int _{0} ^{18C_{*}^{2}n^{-1}(\log n)^{d+2}} + \int _{18C_{*}^{2}n^{-1}(\log n)^{d+2}}   \right) \operatorname{Prob} \{ \left\| \pi _{M} \varphi _{\hat{D},R,n}-\varphi _{\rho} \right\|_{\rho}^{2} > \tau \}  d \tau \\
&\leq  18C_{*}^{2}n^{-1}(\log n)^{d+2} + \int _{0}^{\infty} 3 \exp \left\{  -\frac{m_{1}^{\frac{1}{2}}\tau}{768(4M+2C_fC)^{2}}  \right\} \, d \tau \\
&\leq \mathcal{A}_{6}m_{1}^{-\frac{1}{2}} \left( \log \left( \mathcal{A}_{3}m_{1}^{\frac{1}{2}} \right) \right)^{d+2}
\end{align*}
where $\mathcal{A}_{6}:= 36C_{*}^{2}\mathcal{A}_{3}^{-1}+2304(4M+2C_fC)^{2}.$
This proves the desired bound when the first stage sample size $m_1$ satisfies
$$
 1 \leq \log(8C_{k,d}m_{1}) \leq \frac{3C_{*}^{2}}{4096\mathcal{A}_{4}\mathcal{A}_{3}M^{2}}m_{1}^{\frac{1}{2}} \textrm{ and } \lfloor\mathcal{A}_{3}m_{1}^{\frac{1}{2}}\rfloor \geq 3,$$
 we can see that this restriction on $m_1$ is satisfied when $m_1 \ge \mathcal{A}_7$ where $\mathcal{A}_7$ is constant depending on $C_{k,d}, C_{*}, \mathcal{A}_3, \mathcal{A}_4$ and $M$. When $m_1 < \mathcal{A}_7$, we can also easily see that 
$$\mathbb{E}\{ \mathcal{E}(\pi _{M} \varphi _{\hat{D},R,n})-\mathcal{E}(\varphi _{\rho}) \} \leq 4M^{2} \leq 4M^{2} \sqrt{ \mathcal{A}_{7} }m_{1}^{-\frac{1}{2}}\left( \log \left( \mathcal{\mathcal{A}}_{3}m_{1}^{\frac{1}{2}} \right) \right)^{d+2}$$ and thereby the desired bound still holds. This completes the proof of Theorem 5.
\end{proof}

\acks{The work described in this paper was partially supported by
Discovery Project (DP240101919) of the Australian Research Council.}

\section*{Appendix}
\begin{lemma} 
    \textbf{Ascoli-Arzelà theorem:} let $\{ f_{j} \}$ be a sequence of continuous functions from a separable topological spaces $(X, T)$ into a metric space $(Y,d_{Y})$. Assume that the functions $\{ f_{j} \}$ are equi-bounded and equi-continuous at each $x \in X$. Then, there exists a subsequence $\{ f_{j'} \}\subset \{ f_{j} \}$ and a continuous function $f: X \to Y$ such that $\{ f_{j'} \} \to f$ pointwise in $X$. Moreover the convergence is uniform on compact subsets of $X$.
\end{lemma}

\begin{lemma}
    \label{lem:1} \citep{sriperumbudur2010hilbert}.
    Let $(\Omega,m)$ be a compact metric space. If $k$ is universal, then $\gamma _{k}$ metrizes the weak topology on $P(\Omega)$.
\end{lemma}

\begin{proof}
    We need to show the topology induced by $\gamma _{k}$ is equivalent to the weak topology, i.e., for measures $\{\mathcal{P}_{n}\} \subset P(\Omega)$, $\mathcal{P}_{n} \stackrel{w}{\rightarrow} \mathcal{P}$ iff $\gamma _{k}(\mathcal{P}_{n},\mathcal{P}) \to 0$ as $n \to \infty$. 
First, since $k$ is universal, $\mathcal{H}_{k}$ is dense in $C_{}(\Omega)$. Then for every $g \in C_{}(\Omega)$ and every $\epsilon>0$, there exists a $g' \in \mathcal{H}_{k}$ such that $\left\| g-g' \right\|_{\infty} \leq \epsilon$. Therefore, we have 
\begin{align*}
\left| \int  _{\Omega}g \, d\mathcal{P}_{n} - \int  _{\Omega}g \, d\mathcal{P}   \right| &=\left| \int  _{\Omega} (g-g') \, d\mathcal{P}_{n} + \int  _{\Omega}g' \, d(\mathcal{P}_{n}-\mathcal{P}) + \int  _{ \Omega}(g'-g) \, d\mathcal{P}   \right| \\
&\leq \int _{\Omega} \left| g-g' \right| \, d\mathcal{P}_{n} + \int _{\Omega}  \left| g'-g \right| \, d\mathcal{P} + \left\| g' \right\|_{\mathcal{H}_{k}} \gamma _{k} (\mathcal{P}_{n},\mathcal{P})   \\
& \leq 2 \epsilon+ \left\|  g' \right\|_{\mathcal{H}_{k}} \gamma _{k}(\mathcal{P}_{n},\mathcal{P}).
\end{align*}

Since $\epsilon$ can be arbitrarily small, $\gamma _{k}(\mathcal{P}_{n},\mathcal{P}) \to 0$ implies that $\mathcal{P}_{n} \stackrel{w}{\rightarrow}\mathcal{P}$.
It's also trivial that $\mathcal{P}_{n} \stackrel{w}{\rightarrow}\mathcal{P}$ implies $\gamma _{k}(\mathcal{P}_{n},\mathcal{P})\to 0$, since $\left\| g \right\|_{\infty} \leq \left\| k \right\|_{\infty}\left\|  g \right\|_{\mathcal{H}_{k}}$ for $g \in \mathcal{H}_{k}$.
\end{proof}

\end{document}